\documentclass[11pt, english]{article}
\usepackage{amsmath}
\usepackage{amssymb}
\usepackage{amsfonts}
\usepackage{accents}
\usepackage{dsfont}
\usepackage{appendix}

\usepackage[mathscr]{euscript}

\usepackage{enumitem}
\usepackage[table,xcdraw]{xcolor}

\newcolumntype{x}[1]{>{\centering\let\newline\\\arraybackslash\hspace{0pt}}m{#1}}
\definecolor{DarkBlue}{rgb}{0.1,0.1,0.5}
\definecolor{DarkGreen}{rgb}{0.1,0.5,0.1}
\usepackage[colorlinks=true,linkcolor=blue]{hyperref}
\hypersetup{
	colorlinks   = true,
	linkcolor    = DarkBlue, 
	urlcolor     = DarkBlue, 
	citecolor    = DarkGreen 
}
\usepackage[]{algorithm}
\usepackage{algorithmicx}
\usepackage[noend]{algpseudocode}

\usepackage{microtype}
\usepackage{graphicx}
\usepackage{caption}
\usepackage[skip=0cm,list=true,labelfont=it]{subcaption}

\usepackage{booktabs}

\usepackage{amsthm,thmtools,thm-restate}
\newtheoremstyle{thmstyle}
{0.5em} 
{0.15em} 
{} 
{} 
{\bfseries} 
{.} 
{.5em} 
{} 

\theoremstyle{thmstyle} 
\newtheorem{thm}{Theorem}
\newtheorem{lem}{Lemma}
\newtheorem{propn}{Proposition}[section]
\newtheorem{claim}{Claim}[section]
\newtheorem{cor}{Corollary}

\theoremstyle{definition}
\newtheorem{defn}{Definition}

\theoremstyle{remark}
\newtheorem{rem}{Remark}

\newcommand{\comment}[1]{}

\onecolumn

\RequirePackage{fancyhdr}

\relax
\newlength\titlebox 
\def\addcontentsline#1#2#3{}

\skip\footins 9pt plus 4pt minus 2pt
\def\footnoterule{\kern-3pt \hrule width 5pc \kern 2.6pt }
\labelwidth\leftmargini\advance\labelwidth-\labelsep \labelsep 5pt

\belowdisplayskip \abovedisplayskip
\def\toptitlebar{
\hrule height4pt
\vskip .25in}

\def\bottomtitlebar{
\vskip .25in
\hrule height1pt
\vskip .25in}
\DeclareMathOperator*{\argmax}{arg\,max}
\DeclareMathOperator*{\argmin}{arg\,min}

\newcommand{\1}{ \mathds{1}}

\newcommand{\E}{\mathbb{E}}

\newcommand{\I}{\mathcal{I}}
\newcommand{\J}{\mathcal{J}}

\newcommand{\N}{\mathbb{N}}

\renewcommand{\O}{\mathcal{O}}
\renewcommand{\P}{\mathcal{P}}

\newcommand{\prob}{\mathbb{P}}

\newcommand{\R}{\mathbb{R}}
\renewcommand{\S}{\mathcal{S}}

\newcommand{\Regret}{{\rm Regret}}

\newcommand{\CE}{\texttt{ALG-CE}}
\newcommand{\UE}{\texttt{ALG-UE}}

\newcommand{\tr}{\top}

\newcommand{\event}[1]{\textrm{E}_{#1}}

\renewcommand{\baselinestretch}{1.1}

\title{\hsize\textwidth \linewidth\hsize \toptitlebar {\centering
{\Large\bfseries Intervention Efficient Algorithm for Two-Stage Causal MDPs}}
\bottomtitlebar \vskip 0.2in}

\author{Rahul Madhavan\thanks{Indian Institute of Science. {\tt mrahul@iisc.ac.in}} \qquad Aurghya Maiti\thanks{Adobe Research. {\tt aurghya.kgp@gmail.com}} \qquad Gaurav Sinha\thanks{Adobe Research. {\tt gasinha@adobe.com}} \qquad Siddharth Barman\thanks{Indian Institute of Science. {\tt barman@iisc.ac.in}}}
\date{\vspace{-2.5ex}}
\begin{document}
\maketitle
\begin{abstract}
We study Markov Decision Processes (MDP) wherein states correspond to causal graphs that stochastically generate rewards. In this setup, the learner's goal is to identify atomic interventions that lead to high rewards by intervening on variables at each state.
Generalizing the recent causal-bandit framework, the current work develops (simple) regret minimization guarantees for two-stage causal MDPs, with parallel causal graph at each state.
We propose an algorithm that achieves an instance dependent regret bound. A key feature of our algorithm is that it utilizes convex optimization to address the exploration problem. We identify classes of instances wherein our regret guarantee is essentially tight, and experimentally validate our theoretical results.
\end{abstract}

\section{Introduction}
	Recent years have seen an active interest in causal reinforcement learning. In this thread of work, a fundamental model is that of \emph{causal bandits} \cite{Bareinboim2015, Lattimore, Sen2017, LeeBarenBoim2018, Yabe2018, LeeBareinboim2019, Gaurav2020}. In the causal bandits setting, one assumes an environment comprising of causal variables that influence an outcome of interest; specifically, a reward. The goal of a learner then is to maximize her reward by \emph{intervening} on certain variables (i.e., by fixing the values of certain variables). Note that the reward is assumed to be dependent on the values that the causal variables take, and the causal variables themselves may influence each other. The relationship between these causal variables is typically expressed via a directed acyclic graph (DAG), which is referred to as the causal graph \cite{PearlBook}.
	
	Of particular interest are causal settings wherein the learner is allowed to perform \emph{atomic interventions}. Here, at most one causal variable can be set to a particular value, while other variables take values in accordance with their underlying distributions. Prominent results in the context of atomic interventions include \cite{Correa_Bareinboim_2020} and \cite{Bhattacharya2020}.

	It is relevant to note that when a learner performs an intervention in a causal graph, she gets to see the values that all the causal variables took. Hence,  the collective dependence of the reward on the variables is observed through each intervention. That is, from such an observation, the learner may be able to make inferences about the (expected) reward under other values for the causal variables \cite{PetersBook}. In essence, with a single intervention, the learner is allowed to intervene on a variable (in the causal graph), allowed to observe all other variables, and further, is privy to the effects of such an intervention. Indeed, such an observation in a causal graph is richer than a usual sample from a stochastic process. Hence, a standard goal in causal bandits (and causal reinforcement learning, in general) is to understand the power and limitations of interventions. This goal manifests in the form of developing algorithms that identify intervention(s) that lead to high rewards, while using as few observations/interventions as possible. We use the term \emph{intervention complexity} (rather than sample complexity) for our algorithm, to emphasize the point that in causal reinforcement learning one deals with a richer class of observations. 
	
	Addressing causal bandits, the notable work of Lattimore et al.~\cite{Lattimore} obtains an intervention-complexity bound with a focus on atomic interventions and parallel causal graphs. 
    While the causal bandit framework provides a meaningful (causal) extension of the classic multi-armed bandit setting, it is still not general enough to directly capture settings wherein one requires multiple states to model the environment. Specifically, causal bandits, as is, do not carry the save modelling prowess as Markov decision processes (MDP). Motivated, in part, by this consideration, recent works in causal reinforcement learning have generalized causal bandits to causal MDPs, see, e.g.,  \cite{Ambuj2020}.     
	
	The current work contributes to this thread of work and extends the causal bandit framework of Lattimore et al.~\cite{Lattimore}. In particular, we develop results for two-stage causal MDPs (see \hyperref[figure: two-stage MDP]{Figure \ref{figure: two-stage MDP}}). Such a setup is general enough to address the fact that underlying environment states can evolve, as in an MDP, while simultaneously utilizing (causal) aspects from the causal bandit setup.
	
	We now provide a stylized example that highlights the applicability of the two-stage model. Consider a patient who visits a doctor with certain medical issues. The patient may arrive with a combination of symptoms and lifestyle factors. Some of these may include immediate symptoms, such as fever, but may also include more complex lifestyle factors, e.g., a sedentary routine or smoking. On observing the patient and before prescribing an invasive procedure, the doctor may consider prescribing certain lifestyle changes or milder medicines. This initial intervention can then lead to the patient evolving to a new set of symptoms. At this point, with fresh symptoms and lifestyle factors (i.e., in the second stage of the MDP), the doctor can finalize a course of medication. Such an interaction can be modelled as a two-stage causal MDP, and is not directly captured by the causal bandit framework. Also, the outcome of whether the patient is cured, or not, corresponds to a 0-1 reward for the interventions chosen by the doctor.

    \subsection{Additional Related Work}
    An extension to the earlier literature on causal bandits---towards causal MDPs---was proposed by Lu et al.~\cite{Ambuj2020}. This work considers a causal graph at each state of the MDP. Furthermore, in this model, along with the rewards, the state transitions are also (stochastically) dependent on the causal variables. We address a similar model in our two-stage causal MDP, wherein the state transitions as well as the rewards are functions of the causal variables. 
    It is, however, relevant to note that in \cite{Ambuj2020} it is assumed that the MDP can be initialized to any state. The work of Azar et al.~\cite{azar2017minimax} also conforms to this assumption. Hence, while these two results address a more general MDP setup (than the two-stage one), their results are not directly applicable in the current context wherein the MDP always starts at a specific state and transitions based on the chosen interventions. Indeed, the assumption that the MDP can be initialized arbitrarily might not hold in real-world domains, such as the medical-intervention example mentioned above.   
    
    Sachidananda and Brunskill \cite{Sachidananda2017} propose a Thompson-Sampling based model in a causal bandit setting to minimize cumulative regret. Nair et al.~\cite{Gaurav2020} study the problem of online causal learning to minimize expected cumulative regret under the setting of no-backdoor graphs. They also supply an algorithm for expected simple regret minimization in the causal bandit setting with non-uniform costs associated with the interventions.
    
    Much of the literature in causal learning assumes the causal graph structure is known. In more general settings, learning the causal graph structure is an important sub-problem; for relevant contributions to the problem of causal graph learning see  \cite{Shanmugam2015,Shanmugam2017,Murat2017}, and references therein. Lu et al.~\cite{lu2021causal} and Maiti et al.~\cite{maiti2021causal} extend this to the causal bandit problem. Further, under many circumstances, the structure of the causal graph can be learnt externally, or via some form of hypothesis testing \cite{Arnab2018}.
    
	The current work contributes to the growing body of work on causal reinforcement learning by developing intervention-efficient algorithms for finding near-optimal policies. We focus on simple regret minimization (i.e., near optimal policy identification) in causal MDPs.
	
	\subsection{Our Contributions}
	
	Our main contributions are summarized next. 
	
	We formulate and study two-stage causal MDPs, which encompass many of the issues that arise when considering extensions from bandits to general MDPs. At the same time, the current setup is structured enough to be amenable to a thorough analysis. 
	A notable feature of our setting is that we do not assume that the learner has ready access to all the states, and has to rely on the transitions to reach certain states. 
	
    Here, we develop and analyze an algorithm for finding (near) optimal intervention policies. The algorithm's objective is to minimize simple regret in an intervention efficient manner. We focus on causal MDPs wherein the nonzero transition probabilities are sufficiently high and show that, interestingly, the intervention complexity of our algorithm depends on an instance dependent structural parameter---referred to as $\lambda$ (see equation (\ref{eqn:lambda}))--- rather than directly on the number of interventions or states (Theorem \ref{theorem:main}).
	
	Notably, our algorithm uses a convex program to identify optimal interventions. Using convex optimization to design efficient explorations is a distinguishing feature of the current work. The algorithm spends some time of the given budget learning the MDP parameters (e.g., the transition probabilities). After this, it solves an optimization problem to design efficient exploration of the causal graphs at various states. Such an optimization problem gives rise to the structural parameter, $\lambda$, of the causal MDP instance. We note that the parameter $\lambda$ can be significantly smaller than, say, the total number of interventions in the causal MDP, as demonstrated by our experiments (see Section \ref{section: experiments}).

	In fact, we provide a lower bound showing that our algorithm's regret guarantee is tight (up to a log factor) for certain classes of two-stage causal MDPs (see Section \ref{section: lower bounds}).
	
	\section{Notation and Preliminaries}
	\label{section: notation and preliminaries}
	We consider a Markov decision process (MDP) that starts at state $0$, transitions to one of $k$ states $[k] =\{1,\dots,k\}$, receives a reward, and then finally terminates at state $S_t$; see \hyperref[figure: two-stage MDP]{Figure \ref{figure: two-stage MDP}}. At each state $\{0,1,\dots,k\}$, there is a causal graph along the lines of the ones studied in \cite{Lattimore}; see \hyperref[figure: causal graphs]{Figure \ref{figure: causal graphs}}. In particular, at state $i \in \{0,1,\ldots, k\}$, the causal graph is composed of $n$ independent Bernoulli  variables $\{X^i_1,\dots,X^i_n\}$. For each $X^i_j \in \{0,1\}$, the associated probability $q^i_j := \prob\left\{X^i_j = 1\right\}$. 
	
	\begin{figure}[t]
     \centering
     \begin{subfigure}[t]{0.48\textwidth}
         \centering
         \includegraphics[width=\textwidth]{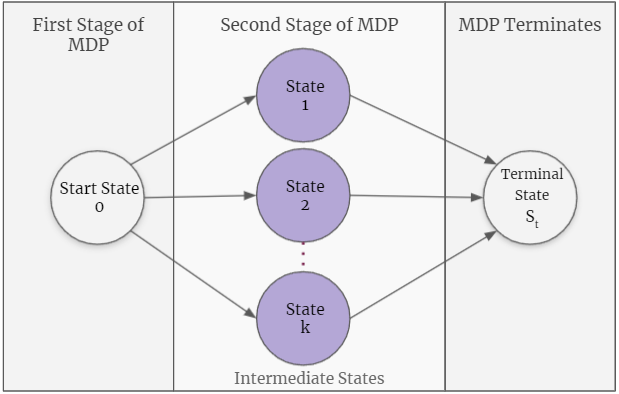}
         \caption{Illustrative figure for two-stage MDP}
         \label{figure: two-stage MDP}
     \end{subfigure}
     \hfill
     \begin{subfigure}[t]{0.48\textwidth}
         \centering
         \includegraphics[width=\textwidth]{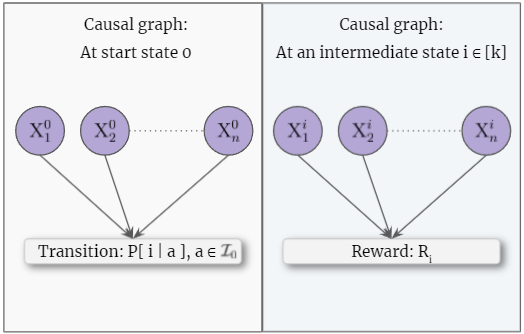}
         \caption{Illustrative Figure for Causal Graph at start state and at some intermediate stage $i\in[k]$}
         \label{figure: causal graphs}
     \end{subfigure}
        \caption{Illustrations for Causal MDP}
        \label{fig:Illustrations for setup}
    \end{figure}
\comment{
	\begin{figure}[t]
		\includegraphics[width=0.5\linewidth]{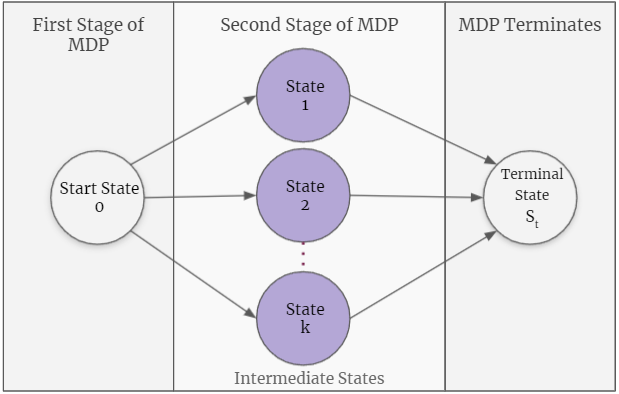}
		\caption{Illustrative figure for two-stage MDP}
		\label{figure: two-stage MDP}
	\end{figure}
	\begin{figure}[t]
		\includegraphics[width=0.5\linewidth]{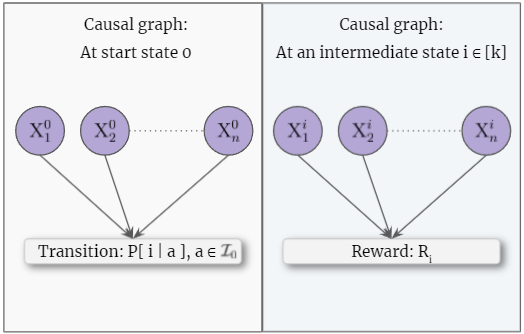}
		\caption{Illustrative Figure for Causal Graph at start state and at some intermediate stage $i\in[k]$}
		\label{figure: causal graphs}
	\end{figure}
}
	In the MDP, for each state $i \in \{0,\dots, k\}$, all the variables $X^i_1, \dots, X^i_n$, are observable. Furthermore, we are allowed atomic interventions, i.e., we can select \emph{at most} one variable and set it to either $0$ or $1$. We will use $\I_i$ to denote the set of atomic interventions available at state $i\in \{0, \ldots,k\}$; in particular, $\I_i = \left\{do() \right\} \cup \left\{do(X^i_j=0), do(X^i_j=1) \ : \ j \in [n] \right\}$. We note that $do()$ is an empty intervention that allows all the variables to take values from their underlying (Bernoulli) distributions. Also, $do(X^i_j=0)$ and $do(X^i_j=1)$ set the value of variable $X^i_j$ to $0$ and $1$, respectively, while leaving all the other variables to independently draw values from their respective distributions. Note that for all $ i \in [k]$, we have $\lvert \I_i \rvert = 2n+1$. Write $N:= 2n+1$.
	
	The model provides us with a $\{0,1\}$ reward as we transition to the terminal state $S_t$ from an intermediate state. Depending on the state $i\in[k]$, from where we transition to $S_t$, we label the reward as $R_i$. Note that the reward $R_i$ stochastically depends on the variables $X^i_1,\dots, X^i_n$; in particular, for all $j \in [n]$ and each realization $X^i_j = x_j \in \{0,1\}$, the reward $R_i$ is distributed as $\prob\left\{R_i=1 \mid X^i_1 = x_1, \ldots, X^i_n = x_n \right\}$. Extending this, we will write $\E[R_i \mid a]$ to denote the expected value of reward $R_i$ when intervention $a \in \I_i$ is performed in state $i \in [k]$. For instance, $\E[R_i \mid do(X^i_j=1)]$ is the expected reward when variable $X^i_j$ is set to $1$, and all the other variables independently draw values from their respective distributions.

	Note that, across the states, the probabilities $q^i_j$s and the reward distributions are fixed but unknown. Indeed, the high-level goal of the current work is to develop an algorithm that---in a sample efficient manner---identifies interventions that maximize the expected rewards.
	
	We denote by $m_i$, the causal parameter from \cite{Lattimore} at state $i$. This parameter is a crucial factor in the regret bound obtained by \cite{Lattimore}. Formally, at state $i$, we consider the Bernoulli probabilities of the variables in increasing order, $q^i_{(1)} \leq q^i_{(2)} \leq \ldots \leq q^i_{(n)}$, and write $m_i := \max \left\{j \enspace\mid\enspace q^i_{(j)} < {1}/{j} \right\}$. In addition, let $M\in\N^{k\times k}$ denote the diagonal matrix of $m_1, \ldots, m_k$.  
	\begin{rem}
		\label{rem: computing mi values}
		The probabilities $q^i_j$s are a priori unknown. It is, however, instructive to consider the computation of $m_i$ from $q^i_j$s: (1) Without loss of generality, assume that $\prob\left\{X_j^i = 1\right\} \leq \prob\left\{X_j^i = 0\right\}$ (otherwise consider the lesser of the two quantities as $q_j^i$) (2) Sort the $q_j^i$s in increasing order (3) Compute $m_i = \max\{ j \enspace\mid\enspace q_j^i < 1/j \}$ and write $\I_{m_i} := \{do(X_j^i=1) \enspace:\enspace  q_j^i < 1/j \}$
	\end{rem}

	\textbf{MDP Notations:} At state $0$, the transition to the intermediate states $[k]$ stochastically depends on the independent Bernoulli random variables $\{X^0_1,\dots,X^0_n\}$. Here, $\prob \{i \enspace\mid\enspace a\}$ denotes the probability of transitioning into state $i \in [k]$ with atomic intervention atomic intervention $a \in \I_0$; recall that $\I_0$ includes the do-nothing intervention. We will collectively denote these transition probabilities as matrix $P:= \left[ P_{(a,i)} = \prob \{ i \enspace\mid\enspace a \} \right]_{a \in \I_0 , i\in[k]}$. Furthermore, write $p_+$ to denote the minimum non-zero value in $P$. Note that matrix $P \in \R^{\lvert \I_0 \rvert \times k}$ is fixed, but unknown.

	A map $\pi:\{0,\dots,k\} \to \I$, between states and interventions (performed by the algorithm), will be referred to as a policy. Specifically, $\pi(i) \in \I_i$ is the intervention at state $i\in\{0,1, \ldots, k\}$. Note that, for any policy ${\pi}$, the expected reward is equal to $\sum_{k=1}^n \E \left[R_i \ \mid \ {\pi}(i) \right] \cdot \prob \{ i \ \mid \ \pi(0) \}$. 
	
	Maximizing expected reward, at each intermediate state $i \in [k]$, we obtain the overall optimal policy $\pi^*$ as follows: $\pi^*(i) = \argmax_{a \in \I_i} \ \E \left[ R_i \mid a\right]$, for $i \in [k]$, and $\pi^*(0) = \argmax_{b \in\I_0} \left( \sum_{i=1}^k \  \E \left[R_i\mid \pi^*(i) \right] \cdot \prob \{ i\mid b \} \right)$.
	
	Our goal is to find a policy $\pi$ with (expected) reward as close to that of $\pi^*$ as possible. We will use $\varepsilon(\pi)$ to denote the sub-optimality of a policy $\pi$; in particular, $\varepsilon(\pi)$ is defined as the difference between the expected rewards of $\pi^*$ and $\pi$.  
	
	Conforming to the standard \emph{simple-regret} framework, the algorithm is given a time budget $T$, i.e., the algorithm can go through the two-stages of the MDP $T$ times. In each of these $T$ rounds, the algorithm can perform the atomic interventions of its choice (both at state $0$ and then at the resulting intermediate state). The overall goal of the algorithm is to compute a policy with high expected reward and the algorithm's sub-optimality is defined as its regret, $\Regret_T := \E [ \varepsilon (\widehat{\pi})]$. Here, the expectation is with respect to the policy $\widehat{\pi}$ computed by the algorithm; indeed, given any two-stage causal MDP instance and time budget $T$ to an algorithm, different policies $\widehat{\pi}$s will have potentially different probabilities of being returned.

	\begin{table}[ht]
	\small
		\caption{MDP Notations} \label{table: MDP notations table}
		\begin{center}
			\begin{tabular}{|x{1.7cm}|x{5.75cm}|}
				\hline\hline
				Notation & Explanation\\
				\hline
				$P\in\R^{N\times k}$& Transition probabilities matrix: $P:= \left[ P_{(a,i)} = \prob\{ i \enspace \mid \enspace a \} \right]_{a \in \I_0, i\in[k] }$\\
				\hline
				$p_+$ & $p_+ = \min\{P_{(a,i)} \mid  P_{(a,i)} > 0\}$\\
				\hline
				$\pi:\S \to \I$ &Policy, a map from states to interventions.\newline i.e. $\pi(i)\in\I_i$ for $i\in \{0\}\cup [k]$\\
				\hline
				$\E\left[R_i  \mid  \pi(i) \right]$ & Expectation of the reward at state $i$ given intervention $\pi(i)$  \\
				\hline
				$\pi^*$&Optimal Policy \\ 
				$\widehat{\pi}$& Computed policy\\
				\hline
				$\varepsilon(\pi)$ & Sub-optimality of $\pi$ \\
				$\Regret_T$ & $\E[\varepsilon(\widehat{\pi})]$ \\
				\hline
				\hline
			\end{tabular}
		\end{center}
	\end{table}
	
	\section{Main Algorithm and its Analysis}
	Our algorithm (\hyperref[alg:best policy generator]{\CE}) uses subroutines to estimate the transition probabilities, the causal parameters,  and the rewards. From these, it outputs the best available interventions as its policy $\widehat{\pi}$. Given time budget $T$, the algorithm uses the first $T/3$ rounds to estimate the transition probabilities (i.e., the matrix $P$) in \hyperref[alg:estimateTransitionProbabilities]{Algorithm \ref{alg:estimateTransitionProbabilities}}. The subsequent $T/3$ rounds are utilized in \hyperref[alg:estimateCausalParameters]{Algorithm \ref{alg:estimateCausalParameters}} to estimate  causal parameters $m_i$s. Finally, the remaining budget is used in \hyperref[alg:estimateRewards]{Algorithm \ref{alg:estimateRewards}} to estimate the intervention-dependent reward $R_i$s, for all intermediate states $i\in[k]$.
	
	\begin{figure}[!t]
	\begin{minipage}[t]{0.99\textwidth}
	\begin{minipage}[t]{0.48\textwidth}
	\begin{algorithm}[H]
	\small
		\caption{$\CE$: Convex Exploration\\ Algorithm}
		\label{alg:best policy generator}
		\begin{algorithmic}[1] 
			\State \textbf{Input:} Total rounds $T$
			\State $\widehat{P} \gets \hyperref[alg:estimateTransitionProbabilities]{\texttt{Estimate Transition Probabilities}}(\frac{T}{3})$
			\State $ \tilde{f} \gets \argmax\limits_{\text{fq.~vector } f } \enspace \min\limits_{\text{states [k]}}  \widehat{P}^\tr  f $
			\State $\hat{M} \gets \hyperref[alg:estimateCausalParameters]{\texttt{Estimate Causal Parameters}}(\tilde{f},\frac{T}{3})$
			\State $\widehat{f}^* \gets \argmin\limits_{\text{fq.~vector } f } \enspace \max\limits_{\text{interventions } \mathcal{I}_0} \widehat{P}\hat{M}^{1/2}\left(\widehat{P}^\tr f \right)^{\circ-\frac{1}{2}}$ \label{step:fstar}
			\Statex $\triangleright$ This is a \hyperref[lem:optimization problem is convex]{Convex Program}
			\vspace{0.05in}
			\State $\widehat{\mathcal{R}} \gets \hyperref[alg:estimateRewards]{\texttt{Estimate Rewards}}(\widehat{f}^*,\tilde{f},\frac{T}{3})$
			\State Estimate optimal policy $\widehat{\pi}(i)\enspace \forall i \in [k]$ based on $\widehat{\mathcal{R}}$
			\State Estimate $\widehat{\pi}(0)$ from $\widehat{P}$ and $\widehat{\pi}(i)\enspace \forall i\in[k]$
			\State \textbf{return} $\widehat{\pi} = \{\widehat{\pi}(0),\widehat{\pi}(1),\dots,\widehat{\pi}(k)\}$
		\end{algorithmic}
	\end{algorithm}
	\end{minipage}
    \hfill
    \begin{minipage}[t]{0.48\textwidth}
	\begin{algorithm}[H]
	\small
		\caption{Estimate Transition Probabilities}
		\label{alg:estimateTransitionProbabilities}
		\begin{algorithmic}[1] 
			\State \textbf{Input:} Time budget $T'$
			\For{time $t\gets \{1,\dots,\frac{T'}{2}\}$}
			\State Perform $do()$ at state $0$. Transition to $i\in [k]$
			\State Count number of times state $i\in[k]$ is observed
			\State Update $\widehat{q}_j^0 = \prob\left\{X_j^0 = 1\right\}$ 
			\EndFor
			\State Using $\widehat{q}_j^0$s, estimate\footnotemark[1] $m_0$ and the set $\I_{m_o}$		
			\Statex $\triangleright$ Note that $\I_0 = \I_{m_0}\cup \I_{m_0}^c$
			\State Estimate $\widehat{P}_{(a,i)}=\prob[i \mid  a]$ for each $a \in \I_{m_0}^c,\enspace i\in[k]$
			\For{intervention $a \in \I_{m_o}$ at state $0$}
			\For{time $t \gets \{1,\dots \frac{T'}{2 \lvert\I_{m_0}\rvert}\}$}
			\State Perform $a\in\I_{m_o}$. Transition to some $i\in[k]$
			
			\label{step:alg2step10}
			\State Count number of times state $i$ is observed
			\EndFor
			\EndFor
			\State Estimate $\widehat{P}_{(a,i)}=\prob[i \mid  a]$ for each $a \in \I_{m_0},\enspace i\in[k]$
			\State \textbf{return} Estimated matrix $\widehat{P}=\left[ \widehat{P}_{(a,i)}\right]_{i \in [k],a \in \I_0}$
	\end{algorithmic}
	\end{algorithm}
	\end{minipage}
	\end{minipage}
	\end{figure}
	
	To judiciously explore the interventions at state $0$, \hyperref[alg:best policy generator]{\CE} computes frequency vectors $f\in\R^{\lvert  \I_0 \rvert}$. In such vectors, the $a$th component $f_a \geq 0$ denotes the fraction of time that each intervention $a \in \I_0$ is performed by the algorithm, i.e., given time budget $T'$, the intervention $a$ will be performed $f_a T'$ times. Note that, by definition, $\sum_a f_a = 1$ and the frequency vectors are computed by solving convex programs over the estimates. The algorithm and its subroutines throughout consider empirical estimates, i.e., find the estimates by direct counting. Here, let $\widehat{P}$ denote the computed estimate of the matrix $P$ and $\hat{M}$ be the estimate of the diagonal matrix $M$.  We obtain a regret upper bound via an optimal frequency vector $\widehat{f}^*$ (see Step \ref{step:fstar} in \hyperref[alg:best policy generator]{\CE}).

	Recall that for any vector $x$ (with non-negative components), the Hadamard exponentiation ${\circ-\frac{1}{2}}$ leads to the vector $y = x^{\circ-\frac{1}{2}}$ wherein $y_i = \frac{1}{\sqrt{x_i}}$ for each component $i$.  
	
	We next define a key parameter $\lambda$ that specifies the regret bound in Theorem \ref{theorem:main} (below).
	\begin{align}
		\lambda := \min_{\text{fq.~vector} f} \ \left\lVert PM^{1/2}  \left(P^\tr  f \right)^{\circ-\frac{1}{2}} \right\rVert_\infty^2 \label{eqn:lambda}
	\end{align}
	Furthermore, we will write $f^*$ to denote the optimal frequency vector in equation (\ref{eqn:lambda}). Hence, with vector $\nu := P{M}^{1/2}  \left(P^\tr  f^* \right)^{\circ-\frac{1}{2}}$, we have $\lambda = \max_a  \nu_a^2$. Note that  Step \ref{step:fstar} in \hyperref[alg:best policy generator]{\CE} addresses an analogous optimization problem, albeit with the estimates $\widehat{P}$ and $\hat{M}$. Also, we show in \hyperref[lem:optimization problem is convex]{Lemma \ref{lem:optimization problem is convex}} (see Section \ref{section: nature of optimization problems}) that this optimization problem is convex and, hence, Step \ref{step:fstar} admits an efficient implementation. 
	
	\footnotetext[1]{See \hyperref[rem: computing mi values]{Remark \ref{rem: computing mi values}
	}}
	\addtocounter{footnote}{1}

    \begin{figure}[!t]
	\begin{minipage}[t]{0.99\textwidth}
	\begin{minipage}[t]{0.48\textwidth}
	\begin{algorithm}[H]
	\small
		\caption{Estimate Causal Parameters}
		\label{alg:estimateCausalParameters}
		\begin{algorithmic}[1] 
			\State \textbf{Input:} Frequency vector $\tilde{f}$ and time budget $T'$
			\State Update $f(a) \gets \frac{1}{2}\left(\tilde{f}(a)+\frac{1}{\lvert  \I_0 \rvert}\right)\enspace \forall a\in\I_0$ \For{intervention $a \in \I_0$}
			\For{time $t \gets \{1,\dots T' \cdot f(a)\}$}
			\State Perform $a\in \I_0$. Transition to some $i\in[k]$
			\State At state $i$, perform $do()$ and observe $X^i_j$s
			\State Update $\widehat{q}_j^i = \prob\left\{X_j^i = 1\right\}$ 
			\EndFor
			\EndFor
			\State Using $\widehat{q}_j^i$s, estimate $\widehat{m}_i$ values\footnotemark[1] for each state $i \in [k]$
			\State \textbf{return} $\hat{M}$, the diagonal matrix of the $\widehat{m}_i$ values
		\end{algorithmic}
	\end{algorithm}
	\end{minipage}
    \hfill
    \begin{minipage}[t]{0.48\textwidth}
	\begin{algorithm}[H]
	\small
		\caption{Estimate Rewards}
		\label{alg:estimateRewards}
		\begin{algorithmic}[1] 
			\State \textbf{Input:} Optimal frequency $f^*$, min-max frequency $\tilde{f}$, and time budget $T'$
			\State Set $f(a) \gets \frac{1}{3}\left(f^*(a)+\tilde{f}(a)+\frac{1}{\lvert  \I_0 \rvert}\right)\enspace \forall a\in\I_0$
			\For{intervention $a \in \I_0$ at state $0$}
			\For{time $t \gets \{1,\dots f(a)\cdot T'/2\}$}
			\State Perform $a\in\I_0$. Transition to some $i\in[k]$
			\State Perform $do()$ at $i\in[k]$.Observe $X^i_j$'s and $R_i$
			\EndFor
			\EndFor
			
			\State Find the set $\I_{m_i} \enspace \forall i\in[k]$ using $q^i_j$ estimates\footnotemark[1]. 
			\Statex $\triangleright\enspace \I_{m_i} \cup \I_{m_i}^c = \I_{i}$
			\State For all $b \in \I_{m_i}^c$, estimate reward $\widehat{\mathcal{R}}_{(b,i)} = \E\left[R_i \mid b\right] $  
			\For{intervention $a \in \I_0$ at state $0$}
			\For{time $t \gets \{1,\dots f(a)\cdot T'/2\}$}
			\State Perform $a\in\I_0$. Transition to some $i\in[k]$
			\State Iteratively perform $b\in\I_{m_i}$ and observe $R_i$
			\Statex $\qquad\quad\triangleright$ round robin over $b\in\I_{m_i}$ across visitations
			\EndFor
			\EndFor
			
			\State Estimate mean reward $\widehat{\mathcal{R}}_{(b,i)} = \E\left[R_i  \mid  b\right] \ \ \forall b \in \I_{m_i}$ 
			\State \textbf{return} $\widehat{\mathcal{R}} = \left[ \widehat{\mathcal{R}}_{(b,i)}\right]_{i \in [k],b \in \I_i}$ 
		\end{algorithmic}
	\end{algorithm}
	\end{minipage}
	\end{minipage}
	
	\end{figure}

	The following theorem is the main result of the current work. It upper bounds the regret of \hyperref[alg:best policy generator]{\CE}. The result requires the algorithm's time budget to be at least 
	\begin{align}
	    \label{eqn: T0}
	    T_0 := \widetilde{O}\left(\frac{N\max(m_i)}{p_+^3}\right)
	\end{align}

	\begin{thm} \label{theorem:main}
		Let parameter $\lambda$ be as defined in equation (\ref{eqn:lambda}). Then, given number of {rounds} $T \geq T_0$, \hyperref[alg:best policy generator]{\CE} achieves regret 
		$$\Regret_T \in \O\left(\sqrt{\frac{1}{T}\max\left\{\lambda,\frac{m_0}{p_+}\right\}\log\left(NT\right)}\right)$$
		
	\end{thm}
	\subsection{Proof of Theorem 1}
	\label{section: proof of theorem 1}
	We prove the theorem, we analyze the algorithm's execution as falling under either \texttt{good event} or \texttt{bad event}, and tackling the regret under each. 
	\begin{defn}
		\label{defn: good events}
		We define five events, $ \event 1 $ to $ \event 5$, the intersection of which we call as \texttt{good event}, $\textrm{E}$, i.e., $\texttt{good event } E \equiv \bigcap_{i\in[5]}  \event i $.
		\begin{enumerate}[label=E\arabic*.]
			\item[$\text{E}_1$:] $\sum\limits_{i=1}^k \lvert \widehat{P}_{(a,i)} - P_{(a,i)}\rvert \leq \frac{ p_+ }{3}\quad \forall a \in \I_0$. That is, for every intervention $a\in \I_0$, the empirical estimate of transition probability in each of \hyperref[alg:estimateTransitionProbabilities]{Algorithms \ref{alg:estimateTransitionProbabilities}},  \hyperref[alg:estimateCausalParameters]{ \ref{alg:estimateCausalParameters}} and \hyperref[alg:estimateRewards]{\ref{alg:estimateRewards}} is good, up to an absolute factor of $p_+/3$. 
			\item[$\text{E}_2$:] Estimate $\widehat{m}_0 \in [\frac{2}{3}m_0,2m_0]$ in \hyperref[alg:estimateTransitionProbabilities]{Algorithm \ref{alg:estimateTransitionProbabilities}}. In other words, our estimate for causal parameter $m_0$ for state 0 in Algorithm \ref{alg:estimateTransitionProbabilities} is relatively good. 
			\item[$\text{E}_3$:] $\widehat{m}_i \in [\frac{2}{3}m_i,2m_i]$, for all states $i\in[k]$. That is,  our estimate of $m_i$ parameter is relatively good for every state $i\in[k]$, in \hyperref[alg:estimateCausalParameters]{Algorithms \ref{alg:estimateCausalParameters}} and  \hyperref[alg:estimateRewards]{\ref{alg:estimateRewards}}.
			\item[$\text{E}_4$:] $\sum_{i\in[k]} \lvert \widehat{P}_{(a,i)} - P_{(a,i)} \rvert  \leq \eta'$, for all interventions $a \in \I_0$. Here, random variable $\eta' := \sqrt{\frac{150m_0}{Tp_+}\log\left(\frac{3T}{k}\right)}$ and $\widehat{P}_{(a,i)}$ is the estimated transition probability computed in \hyperref[alg:estimateTransitionProbabilities]{Algorithm \ref{alg:estimateTransitionProbabilities}}.
			\item[$\text{E}_5$:] $\left\lvert \E\left[R_i  \mid  a \right] - \widehat{\mathcal{R}}_{(a,i)}\right\rvert  \leq \widehat{\eta}_i$  for all $i\in[k]$ and all $a\in\I_i$; here $\widehat{\eta}_i = \sqrt{\frac{27 \widehat{m}_i}{T(\widehat{P}^\tr \widehat{f}^*)_i}\log\left(2TN \right)}$.\footnote{Recall that $\widehat{f}^*$ denotes the optimal frequency vector computed in Step \ref{step:fstar} of \hyperref[alg:best policy generator]{\CE}. Also, $(\widehat{P}^\tr \widehat{f}^*)_i$ denotes the $i$th component of the vector $P^\tr f^*$.} 
		\end{enumerate}
	\end{defn}
	\begin{defn}
		\label{defn: bad event}
		We define \texttt{bad event} $\textrm{F}$, as the complement of the intersection of events $ \event 1 $ - $ \event 5$, as defined \hyperref[defn: good events]{above}, i.e., $\texttt{bad event } F\equiv E^c$.
	\end{defn}
	
	Before we proceed with the proof, we state below a corollary which provides a multiplicative bound on $\widehat{P}$ with respect to $P$, complementing the additive form of $ \event 1 $.
	\begin{cor}
		\label{cor: corollary to E1 as a multiplicative bound}
		Under event $ \event 1 $, for all interventions $a\in\I_0$ and states $i\in[k]$, we have:
		\begin{align*}
		    \frac{2}{3}P_{(a,i)}\leq \widehat{P}_{(a,i)}\leq \frac{4}{3}P_{(a,i)}
		\end{align*}
	\end{cor}
	\begin{proof}
		Event $ \event 1 $ ensures that $\sum\limits_{i=1}^k \lvert \widehat{P}_{(a,i)} - P_{(a,i)} \rvert \leq \frac{ p_+ }{3}$, for each interventions $a\in\I_0$ and states $i\in[k]$. This, in particular, implies that for each intervention $a\in\I_0$ and state $i\in[k]$ the following inequality holds: $ \lvert \widehat{P}_{(a,i)} - P_{(a,i)} \rvert \leq \frac{ p_+ }{3}$. 
		Note that if $P_{(a,i)}=0$, then the algorithm will never observe state $i$ with intervention $a$, i.e., in such a case $\widehat{P}_{(a,i)} = P_{(a,i)} = 0$. 
		For the nonzero $P_{(a,i)}$s, recall that (by definition), $p_+ = \min\{P_{(a,i)}  \mid  P_{(a,i)} > 0\}$. Therefore, for any nonzero $P_{(a,i)}$, the above-mentioned inequality gives us $\lvert \widehat{P}_{(a,i)} - P_{(a,i)} \rvert \leq \frac{1}{3}P_{(a,i)}$. Equivalently, $\widehat{P}_{(a,i)}  \leq \frac{4}{3}P_{(a,i)}$ and $\widehat{P}_{(a,i)}  \geq \frac{2}{3}P_{(a,i)}$. Therefore, for all $P_{(a,i)}$s the corollary holds. 
	\end{proof}

	Considering the estimates $\widehat{P}$ and $\hat{M}$, along with frequency vector $\widehat{f}^*$ (computed in Step \ref{step:fstar}), we define random variable $\widehat{\lambda} := \left\lVert \widehat{P} \hat{M}^{1/2}  \left(\widehat{P}^\tr  \widehat{f}^* \right)^{\circ-\frac{1}{2}}\right\rVert_\infty^2$. Note that $\widehat{\lambda}$ is a surrogate for $\lambda$. We will, in fact, show that, under the good event, $\widehat{\lambda}$ is close to $\lambda$ (Lemma \ref{lem: Bounding lambda hat by lambda}).

	Recall that $\Regret_T := \E [ \varepsilon (\pi)]$ and here the expectation is with respect to the policy ${\pi}$ computed by the algorithm. We can further consider the expected sub-optimality of the algorithm and the quality of the estimates (in particular, $\widehat{P}$, $\hat{M}$ and $\widehat{\lambda}$) under \texttt{good event} (E). 
	
	Based on the estimates returned at Step \ref{step:fstar} of \hyperref[alg:best policy generator]{\CE}, either the \texttt{good event} holds, or we have the \texttt{bad event} (though this is unknown to our algorithm). We obtain the regret guarantee by first bounding sub-optimality of policies computed under the \texttt{good event}, and then bound the probability of the \texttt{bad event}.
	
	\begin{lem}
		\label{lem: difference under pi-star is small}
		For the optimal policy $\pi^*$, under the \texttt{good event} ($\textrm{E}$), we have
		\begin{align*}
		    \sum_{i\in[k]} P_{(\pi^*(0),i)}\E\left[R_i\  \mid \ \pi^*(i) \right]  \leq  \sum \widehat{P}_{(\pi^*(0),i)} \widehat{\mathcal{R}}_{(\pi^*(i),i)} + \O\left(\sqrt{\max\{\widehat{\lambda},\frac{m_0}{p_+}\}\frac{\log\left(NT\right)}{T}}\right)
		\end{align*}
	\end{lem}
	\begin{proof}Consider the expression
	    \begin{align*}\sum_{i\in[k]} P_{(\pi^*(0),i)}\E\left[R_i\  \mid \ \pi^*(i) \right] - \sum_{i\in[k]} \widehat{P}_{(\pi^*(0),i)} \widehat{\mathcal{R}}_{(\pi^*(i),i)}.
	    \end{align*}
	    We can add and subtract $\sum_{i\in[k]} P_{(\pi^*(0),i)}\mathcal{\widehat{R}}_{(\pi^*(i),i)}$ and take common terms out to reduce the expression: 
	    \begin{align*}
	        \sum_{i\in[k]} P_{(\pi^*(0),i)}\left(\E\left[R_i  \mid  \pi^*(i) \right]-\widehat{\mathcal{R}}_{(\pi^*(i),i)}\right) + \sum_{i\in[k]}  \widehat{\mathcal{R}}_{(\pi^*(i),i)} \left(P_{(\pi^*(0),i)}-\widehat{P}_{(\pi^*(0),i)}\right)
	    \end{align*}

		Note that:\vspace{-0.1in}
		\begin{enumerate}[label=(\alph*)]
		    \item $\widehat{\mathcal{R}}_{(\pi^*(i),i)} \leq 1$
		    \item $\sum_{i\in[k]} \lvert \widehat{P}_{(\pi^*(0),i)} - P_{(\pi^*(0),i)}\rvert \leq \eta'$ (from $ \event 4$)
		    \item $\left\lvert \E\left[R_i  \mid  \pi^*(i) \right]-\widehat{\mathcal{R}}_{(\pi^*(i),i)}\right\rvert \leq \widehat{\eta}_i$ (from $ \event 5$)
		\end{enumerate}
		Furthermore, it follows from \hyperref[cor: corollary to E1 as a multiplicative bound]{Corollary \ref{cor: corollary to E1 as a multiplicative bound}} that (component-wise) $P\leq \frac{3}{2}\widehat{P}$.
		Hence, the above-mentioned expression is bounded above by 
		\begin{align*}
		    \sum_{i\in[k]}\frac{3}{2}\widehat{P}_{(\pi^*(0),i)} \widehat{\eta}_i + \eta'
		\end{align*}
		Note that the definition of $\widehat{\lambda}$ ensures $\sum_{i\in[k]}\widehat{P}_{(\pi^*(0),i)} \widehat{\eta}_i = \O\left(\sqrt{\frac{\widehat{\lambda}}{T}\log(NT)}\right)$. Further, $\eta' = \O\left(\sqrt{\frac{m_0}{Tp_+}\log(\frac{T}{k})}\right)$. Therefore,
		\begin{align*}
		    \sum_{i\in[k]}P_{(\pi^*(0),i)} \eta_i  + \eta'= \O\left(\sqrt{\max\{\widehat{\lambda},\frac{m_0}{p_+}\}\frac{\log\left(NT\right)}{T}}\right)
		\end{align*}
		This establishes the lemma.
	\end{proof}
	We now state another similar lemma for any policy $\widehat{\pi}$ computed under \texttt{good event}.  
	\begin{lem}
		\label{lem: difference under widehat-pi is small}
		Let $\widehat{\pi}$ be a policy computed by \hyperref[alg:best policy generator]{\CE} under the \texttt{good event} ($\textrm{E}$). Then, 
		\begin{align*}
		    \sum_{i\in[k]} P_{(\widehat{\pi}(0),i)} \E\left[R_i  \mid  \widehat{\pi}(i)\right] \geq \sum_{i\in[k]}\widehat{P}_{(\widehat{\pi}(0),i)} \widehat{\mathcal{R}}_{(\widehat{\pi}(i),i)}-\O\left(\sqrt{\max\{\widehat{\lambda},\frac{m_0}{p_+}\}\frac{\log\left(NT\right)}{T}}\right)
		\end{align*}
	\end{lem}
	\begin{proof}
		Consider the expression: 
		\begin{align*}
		    \sum_{i\in[k]} \widehat{P}_{(\widehat{\pi}(0),i)} \widehat{\mathcal{R}}_{(\widehat{\pi}(i),i)} - \sum_{i\in[k]} P_{(\widehat{\pi}(0),i)}\E\left[R_i  \mid  \widehat{\pi}(i) \right]
		\end{align*}
		We can add and subtract $\sum_{i\in[k]} P_{(\widehat{\pi}(0),i)}\mathcal{\widehat{R}}_{(\widehat{\pi}(i),i)}$ to get: 
		\begin{align*}
		    \sum_{i\in[k]}  \widehat{\mathcal{R}}_{(\widehat{\pi}(i),i)} \left(\widehat{P}_{(\widehat{\pi}(0),i)} - P_{(\widehat{\pi}(0),i)}\right) + \sum_{i\in[k]} P_{(\widehat{\pi}(0),i)}\left(\widehat{\mathcal{R}}_{(\widehat{\pi}(i),i)}-\E\left[R_i  \mid  \widehat{\pi}(i) \right]\right)
		\end{align*}
		Analogous to \hyperref[lem: difference under pi-star is small]{Lemma \ref{lem: difference under pi-star is small}}, one can show that this expression is bounded above by 
		\begin{align*}
		    \eta' + \sum_{i\in[k]}\frac{3}{2}\widehat{P}_{(\widehat{\pi}(0),i)} \widehat{\eta}_i  = \O\left(\sqrt{\max\{\widehat{\lambda},\frac{m_0}{p_+}\}\frac{\log\left(NT\right)}{T}}\right)
		\end{align*}
	\end{proof}
	
	We can also bound $\widehat{\lambda}$ to within a constant factor of $\lambda$.
	\begin{lem}
		\label{lem: Bounding lambda hat by lambda}
		Under the \texttt{good event}, we have $\widehat{\lambda}\leq 8 \lambda$. 
	\end{lem}
	\begin{proof}
	    \hyperref[cor: corollary to E1 as a multiplicative bound]{Corollary \ref{cor: corollary to E1 as a multiplicative bound}} ensures that given event $ \event 1 $ (and, hence, the \texttt{good event}), $\frac{2}{3}P\leq \widehat{P}\leq \frac{4}{3}P$. In addition, note that event $ \event 3 $ gives us $\hat{M}\leq 2M$. From these observations we obtain the desired bound: 
		\begin{align*}
		    \widehat{\lambda} = \widehat{P}\hat{M}^{\frac{1}{2}}\left(\widehat{P}^\tr \widehat{f}^*\right)^{\circ-\frac{1}{2}} \leq \widehat{P}\hat{M}^{\frac{1}{2}}\left(\widehat{P}^\tr f^*\right)^{\circ-\frac{1}{2}} \leq 8PM^{\frac{1}{2}}\left(P^\tr f^*\right)^{\circ-\frac{1}{2}} = 8\lambda
		\end{align*}
		Here, the first inequality follows from the fact that $\widehat{f}^*$ is the minimizer of the $\widehat{\lambda}$ expression, and for the second inequality, we substitute the appropriate bounds of $\widehat{P}$ and $\hat{M}$.
	\end{proof}

	\begin{cor}
		\label{cor: difference between pi-star and pi-hat gives varepsilon widehat pi}
		Let $\widehat{\pi}$ be a policy computed by \hyperref[alg:best policy generator]{\CE} under \texttt{good event} (E), then
		\begin{align*}
		    \varepsilon(\widehat{\pi}) =  \O\left(\sqrt{\max\left\{{\lambda},\frac{m_0}{p_+}\right\}\frac{\log\left(NT\right)}{T}}\right)
		\end{align*}
	\end{cor}
	\begin{proof}
		Since \hyperref[alg:best policy generator]{\CE} selects the optimal policy (with respect to the estimates), 
		\begin{align*}
		    \sum \widehat{P}_{(\pi^*(0),i)} \widehat{\mathcal{R}}_{(\pi^*(i),i)} \leq \sum \widehat{P}_{(\widehat{\pi}(0),i)} \widehat{\mathcal{R}}_{(\widehat{\pi}(i),i)}
		\end{align*}
		Combining this with \hyperref[lem: difference under pi-star is small]{Lemmas \ref{lem: difference under pi-star is small}} and \hyperref[lem: difference under widehat-pi is small]{\ref{lem: difference under widehat-pi is small}}, we get under \texttt{good event}:
		\begin{align*}
		    \sum_{i\in[k]} P_{(\pi^*(0),i)}\E\left[R_i  \mid  \pi^*(i) \right] - \sum_{i\in[k]} P_{(\widehat{\pi}(0),i)}\E\left[R_i  \mid  \widehat{\pi}(i) \right] = \O\left(\sqrt{\max\left\{\widehat{\lambda},\frac{m_0}{p_+}\right\}\frac{\log\left(NT\right)}{T}}\right)
		\end{align*}
		Note the left-hand-side of this expression is equal to $\varepsilon(\widehat{\pi})$. Finally, using Lemma \ref{lem: Bounding lambda hat by lambda}, we get that 
		\begin{align*}
		    \varepsilon(\widehat{\pi}) = \O\left(\sqrt{\max\left\{\lambda,\frac{m_0}{p_+}\right\}\frac{\log\left(NT\right)}{T}}\right)
		\end{align*}
		The corollary stands proved.  
	\end{proof}

	\hyperref[cor: difference between pi-star and pi-hat gives varepsilon widehat pi]{Corollary \ref{cor: difference between pi-star and pi-hat gives varepsilon widehat pi}} shows that under the \texttt{good event}, the (true) expected reward of $\pi^*$ and $\widehat{\pi}$ are within $\O\left(\sqrt{\max\left\{\lambda,\frac{m_0}{p_+}\right\}\frac{\log\left(NT\right)}{T}}\right)$ of each other.
	
	In Lemma \ref{lem: Bound on F} (stated below and proved in Appendix \ref{section:pr-bad-event}) we will show that\footnote{Recall that, by definition, $\textrm{F} = \textrm{E}^c$.} $\prob\left\{\bigcup_{i\in[5]}\neg  \event i \right\} = \prob\left\{\textrm{F}\right\}\leq \frac{5k}{T}$, for appropriately large $T$.
	
		\begin{restatable}[Bound on \texttt{Bad Event}]{lem}{boundonBadEvent}
    \label{lem: Bound on F}
    Write $T_0 := \O\left(\frac{N\max(m_i)}{p_+^3}\log\left(2NT\right)\right) = \widetilde{O}\left(\frac{N\max(m_i)}{p_+^3}\right)$. Then for any $T>T_0$:
	\begin{align*}
	    \prob\{\textrm{F}\}  \leq \frac{5k}{T}.
	\end{align*}
    \end{restatable}

	The above-mentioned bounds together establish Theorem \ref{theorem:main} (i.e., bound the regret of \hyperref[alg:best policy generator]{\CE}): 
	\begin{align*}
	    \Regret_T = \E[\varepsilon(\pi)] = \E[\varepsilon(\widehat{\pi})  \mid   \textrm{E} ]\enspace \prob\left\{\textrm{E}\right\} \ + \ \E[\varepsilon(\pi')  \mid  \textrm{F}]\enspace\prob\left\{\textrm{F}\right\}
	\end{align*}
	Since the rewards are bounded between $0$ and $1$, we have $\varepsilon(\pi') \leq 1$, for all policies $\pi'$. In addition, the fact that $\prob\{\textrm{E}\}\leq 1$ gives us $\Regret_T \leq \E[\varepsilon(\pi)  \mid   \textrm{E} ] + \prob\{\textrm{F}\}$. Therefore, Corollary \ref{cor: difference between pi-star and pi-hat gives varepsilon widehat pi} along with Lemma \ref{lem: Bound on F}, lead to the stated regret guarantee 
	\begin{align*}
	    \Regret_T = \O\left(\sqrt{\max\left\{\lambda,\frac{m_0}{p_+}\right\}\frac{\log\left(NT\right)}{T}}\right) + \frac{5k}{T} = \O\left(\sqrt{\max\left\{\lambda,\frac{m_0}{p_+}\right\}\frac{\log\left(NT\right)}{T}}\right)
	\end{align*}

	\section{Lower Bound}
    \label{section: lower bounds}
    This section provides a lower bound on regret for a family of instances. For any number of states $k$, we show that there exist transition matrices $P$ and reward distributions ($\E[R_i \mid a]$) such that regret achieved by \CE $\ $(Theorem \ref{theorem:main}) is tight, up to log factors. 
    
    \begin{thm}
    \label{theorem: Lower Bound for our algorithm}
    There exists a transition matrix $P$, reward distributions, and probabilities $q^0_j$ corresponding to causal variables $\{X^0_j\}_{j\in[n]}$, such that for any $q^i_j$, corresponding to causal variables at states $i \in [k]$, the simple regret achieved by any algorithm is 
    \begin{align*}
        \Regret_T\in \Omega\left(\sqrt{\frac{\lambda}{T}}\right)
    \end{align*}
    \end{thm}
    
    \vspace{-0.1in}
    \subsection{Theorem 2: Proof Setup}
    \vspace{-0.1in}
    
    This section establishes Theorem \ref{theorem: Lower Bound for our algorithm}. We will identify a collection of two-stage causal MDP instances and show that, for any given algorithm $\mathscr{A}$, there exists an instance in this collection for which $\mathscr{A}$'s regret is $\Omega\left(\sqrt{\frac{\lambda}{T}}\right)$.

    First we describe the collection of instances and then provide the proof.

	For any integer $k > 1$, consider $n=(k-1)$ causal variables at each state $i \in \{0,1,\dots, k\}$. The transition matrix $P$ is set to be deterministic. Specifically, for each $i \in [n]$, we have $\prob\{i \mid do(X_i^0 = 1)\} = 1$. For all other interventions at state 0, we transition to state k with probability 1. Such a transition matrix can be achieved by setting $q^0_i=0$ for all $i \in [k-1]$. As before, the total number of interventions $N:=2n+1 = 2k-1$.
	
	Now consider a family of $Nk+1$ instances $\left\{\mathcal{F}_0\right\} \cup  \left\{ \mathcal{F}_{(a,i)} \right\}_{i \in [k], a \in \I_i}$. Here, $\mathcal{F}_0$ and each $\mathcal{F}_{(a,i)}$ is a two-stage causal MDP with the above-mentioned transition probabilities. The instances differ in the rewards at the intermediate states. In particular, in instance $\mathcal{F}_0$, we set the reward distributions such that  $\E[R_i \mid a]=\frac{1}{2}$ for all states $i \in [k]$ and interventions $a \in \I_i$. For each $i \in [k]$ and $a \in \I_i$, instance $\mathcal{F}_{(a,i)}$ differs from $\mathcal{F}_0$ only at state $i$ and for intervention $a$. Specifically, by construction, we will have $\E[R_i \mid a] = \frac{1}{2} + \beta$, for a parameter $\beta >0$. The expected rewards under all other interventions will be $1/2$, the same as in $\mathcal{F}_0$.
	
	Given any algorithm $\mathscr{A}$, we will consider the execution of $\mathscr{A}$ over all the instances in the family. The execution of algorithm $\mathscr{A}$ over each instance induces a trace, which may include the realized transition probabilities $\widehat{P}$, the realized variable probabilities $\widehat{q}^i_j$ for $i\in [k]$ and $j \in [n]$ and the corresponding $\widehat{m}_i$s, and the realized rewards $\widehat{\mathcal{R}}$. Each of such realizations (random variables) has a corresponding distribution (over many possible runs of the algorithm). We call the measures corresponding to these random variables under the instances $\mathcal{F}_0$ and $\mathcal{F}_{(a,i)}$ as $\mathcal{P}_0$ and $\mathcal{P}_{(a,i)}$, respectively.

	\vspace{-0.1in}
	\subsection{Proof of Theorem 2}
	\vspace{-0.1in}
	For any algorithm $\mathscr{A}$ and given time budget $T$, we first consider the $\mathscr{A}$'s execution over instance $\mathcal{F}_0$. As mentioned previously, $\mathcal{P}_0$ denotes the trace distribution induced by the algorithm for $\mathcal{F}_0$. In particular, write $r_i$ to denote the expected number of times state $i$ is visited,  $r_i := \E_{\mathcal{P}_0}\left[\text{state $i$ is visited}\right]/T$.

    Recall the construction of the set $\I_{m_i}$ (for each intermediate state $i \in [k]$) from Remark \ref{rem: computing mi values} in Section \ref{section: notation and preliminaries}. In particular, $m_i := \max\{j  \mid  q^i_{(j)}<\frac{1}{j}\}$ and $\I_{m_i} := \{do(X^i_{(j)}=1)  \mid  q^i_{(j)}<\frac{1}{j}\}$, where the Bernoulli probabilities of the variables at state $i$ are sorted to satisfy $q^i_{(1)}\leq q^i_{(2)}\leq \dots \leq q^i_{(n)}$. Note that these definitions do not depend on the algorithm at hand. The algorithm, however, may choose to perform different interventions different number of times. Write $N_{(a,i)}$ to denote the expected (under $\mathcal{P}_0$) number of times intervention $a$ is performed by the algorithm at state $i$. Furthermore, let random variable $T_{(a,i)}$ denote the number of times intervention $a$ is observed at stated $i$. Hence, $\E_{\mathcal{P}_0}[T_{(a,i)}]$ is the expected number of times intervention $a$ is observed.\footnote{Note that $a$ can be observed while performing the do-nothing intervention. Also, the expected value $N_{(a,i)}$ accounts for the number of times $a$ is explicitly performed and not just observed.}

    Using the expected values for algorithm $\mathscr{A}$ and instance $\mathcal{F}_0$, we define a subset of $\I_{m_i}$ as follows: $\J_i := \left\{a\in \I_{m_i}  \ : \  N_{(a,i)} \leq 2\frac{Tr_i}{m_i} \right\}$. The following proposition shows that the size of $\J_i$ is sufficiently large. 
    \begin{propn}
    \label{propn: Size of Ai}
        The set $\J_i$ is non-empty. In particular, 
        \begin{align*}
            m_i/2\leq \lvert \J_i \rvert \leq m_i.
        \end{align*}
    \end{propn}
    \begin{proof}
    The upper bound on the size of subset $\J_i$ follows directly from its definition: since $\J_i\subseteq I_{m_i}$ we have $\lvert \J_i \rvert \leq \lvert \I_{m_i} \rvert = m_i$. 
    
    For the lower bound on the size of $\J_i$, note that $T r_i$ is the expected number of times state $i$ is visited by the algorithm. Therefore, 
    \begin{align}
    \label{ineq:navg}
        \sum_{a \in \I_{m_i}} N_{(a,i)} \leq T r_i
    \end{align}
    
    Furthermore, by definition, for each intervention $b \in \I_{m_i} \setminus \J_i$ we have $N_{(b,i)} \geq \frac{2 T r_i}{m_i}$. Hence, assuming $\lvert \I_{m_i} \setminus \J_i \rvert > \frac{m_i}{2}$ would contradict inequality \ref{ineq:navg}.
    This observation implies that $\lvert \I_{m_i} \setminus \J_i \rvert \leq \frac{m_i}{2}$ and, hence, $\lvert \J_i \rvert \geq \frac{m_i}{2}$. This completes the proof.  
    \end{proof}
    
    Recall that $T_{(a,i)}$ denotes the number of times intervention $a$ is observed at stated $i$. The following proposition bounds $\E[T_{(a,i)}]$ for each intervention $a \in \J_i$. 
    \begin{propn}
    \label{propn: Number of interventions of elements in Ai}
    For every intervention $a \in \J_i$
    \begin{align*}
        \E_{\mathcal{P}_0}[T_{(a,i)}] \leq \frac{3Tr_i}{m_i}.
    \end{align*}
    \end{propn}
    \begin{proof}
        Any intervention $a\in\J_i\subseteq \I_{m_i}$ may be observed either when it is explicitly performed by the algorithm or as a random realization (under some other intervention, including do-nothing). Since $a \in \I_{m_i}$, the probability that $a$ is observed as part of some other intervention is at most $\frac{1}{m_i}$. Therefore, the expected number of times that $a$ is observed by the algorithm---without explicitly performing it---is at most $\frac{T r_i}{m_i}$;\footnote{Here, we use the fact that the realization of $a$ is independent of the visitation of state $i$.} recall that the expected number of times state $i$ is visited is equal to $T r_i$. 
    
    For any intervention $a \in \J_i$, by definition, the expected number of times $a$ is performed $N_{(a,i)} \leq \frac{2 T r_i}{m_i}$. Therefore, the proposition follows: 
        \begin{align*}
            \E[T_{(a,i)}] \leq \frac{T r_i}{m_i} + N_{(a,i)} \leq \frac{3Tr_i}{m_i}.
        \end{align*}
    \end{proof}

    We now state two known results for KL divergence.

			    \textbf{Bretagnolle-Huber Inequality (Theorem 14.2 in \cite{LattimoreBook})}
				\label{eqn: Bretagnolle-Huber Inequality}: 
				Let $\mathcal{P}$ and $\mathcal{P}'$ be any two measures on the same measurable space. Let $\textrm{E}$ be any event in the sample space with complement $\textrm{E}^c$. Then, 
				\begin{align} 
				\prob_{\P}\{E\} + \prob_{\mathcal{P}'}\{E^c\}\geq\frac{1}{2}\exp\left(-\rm{KL}(\mathcal{P},\mathcal{P}')\right). 
				\end{align}

			    \textbf{Bound on KL-Divergence with number of observations (Adaptation of Equation 17 in Lemma B1 from \cite{AuerGamblinginaRiggedCasino})}:
				\label{eqn: Bound on KL Divergence with number of observations}
				Let $\mathcal{P}_0$ and $\mathcal{P}_{(a,i)}$ be any two measures with differing  expected rewards (for exactly the intervention $a$ at state $i$) by an amount $\beta$. Then, 
				\begin{align}
				    \rm{KL}(\mathcal{P}_0,\mathcal{P}_{(a,i)}) \leq 6\beta^2 \  \E_{\mathcal{P}_0}[T_{(a,i)}] \label{ineq:auer}
				\end{align}
	
	We prove the above inequality in \hyperref[appendix section: Proof KL Divergence inequality]{Appendix \ref{appendix section: Proof KL Divergence inequality}}.
	
	Using \hyperref[eqn: Bound on KL Divergence with number of observations]{this bound on KL divergence} and \hyperref[propn: Number of interventions of elements in Ai]{Proposition \ref{propn: Number of interventions of elements in Ai}}, we have, for all states $i\in [k]$ and interventions $a \in \J_i$:  
	\begin{align}
	    \rm{KL}(\mathcal{P}_0,\mathcal{P}_{(a,i)}) \leq 6\beta^2\cdot 3\frac{Tr_i}{m_i} = 18\frac{Tr_i\beta^2}{m_i}
	\end{align}
	
	Substituting this in the \hyperref[eqn: Bretagnolle-Huber Inequality]{Bretagnolle-Huber Inequality}, we obtain, for any event $\textrm{E}$ in the sample space along with all states $i \in [k]$ and all interventions $a \in \J_i$: 
    \begin{align}
    \label{eqn: bound on Prob E under P1 and Prob Ec under P2 using KL}
        \prob_{\P_{(a,i)}}\{E\} + \prob_{\P_0}\{E^c\} \geq\frac{1}{2}\exp\left(-18\frac{Tr_i\beta^2}{m_i}\right)
    \end{align}
    
        We now define events to lower bound the probability that Algorithm $\mathscr{A}$ returns a sub-optimal policy. In particular, write $\widehat{\pi}$ to denote the policy returned by algorithm $\mathscr{A}$. Note that $\widehat{\pi}$ is a random variable. 
    
    For any $\ell \in [k]$ and any intervention $b$, write $G_1(b, \ell)$ to denote the event that---under the returned policy $\widehat{\pi}$---intervention $b$ is \emph{not} chosen at state $\ell$, i.e., $G_1(b, \ell) := \left\{\widehat{\pi}(\ell)\neq b\right\}$. Also, let $G_2(\ell)$ denote the event that policy $\widehat{\pi}$ does not induce a transition to $\ell$ from state $0$, i.e., $G_2(\ell) := \left\{\widehat{\pi}(0)\neq \ell\right\}$. Furthermore, write $G (b, \ell) := G_1(b, \ell) \cup G_2 (\ell)$. Note that the complement $G^c(b, \ell) = G^c_1(b, \ell) \cap G^c_2(\ell) = \{ \widehat{\pi}(\ell) = b \} \cap \{ \widehat{\pi}(0) = \ell \}$.
 
    Considering measure $\mathcal{P}_0$, we note that for each state $\ell \in [k]$ there exists an intervention $\alpha_\ell \in \J_\ell$ with the property that $\prob_{\mathcal{P}_0} \left\{ G^c_1(\alpha_\ell, \ell) \right\} = \prob_{\mathcal{P}_0} \left\{ \widehat{\pi}(\ell) = \alpha_\ell \right\} \leq \frac{1}{\lvert J_\ell \rvert}$. This follows from the fact that $\sum_{a \in \J_\ell}\prob_{\mathcal{P}_0}\left\{\widehat{\pi}(\ell)= a\right\} \leq 1$. Therefore, for each state $\ell \in [k]$ there exists an intervention $\alpha_\ell$ such that $\prob_{\P_0}\{ G^c(\alpha_\ell, \ell) \} \leq \frac{1}{\lvert \J_\ell \rvert}$. 

    This bound and \hyperref[eqn: bound on Prob E under P1 and Prob Ec under P2 using KL]{inequality \ref{eqn: bound on Prob E under P1 and Prob Ec under P2 using KL}} imply that for all states $\ell \in [k]$ there exists an intervention $\alpha_\ell$ that satisfies 
    \begin{align}
    \label{eqn: forall i exists a in Ji such that prob of making mistake is bouded below using KL}
      \prob_{\P_{(\alpha_\ell,\ell)}}\{G(\alpha_\ell, \ell) \} \geq \frac{1}{2}\exp\left(-18\frac{Tr_\ell \beta^2}{m_\ell}\right) - \frac{1}{\lvert \J_\ell \rvert}
    \end{align}
    
    We will set 
    \begin{align}
    \label{eqn: Beta Value}
        \beta = \min\left\{\frac{1}{3},\sqrt{\frac{\sum_{\ell\in[k] } m_\ell }{18T}}\right\}
    \end{align}
    
    Therefore $\beta$ takes value either $\sqrt{\frac{\sum_{\ell\in[k] } m_\ell }{18T}}$ or $\frac{1}{3}$. We will address these over two separate cases.
    
    \textbf{Case 1}: 
    $\beta = \sqrt{\frac{\sum_{\ell\in[k] } m_\ell }{18T}}$.
    
    We wish to substitute this $\beta$ value in 
    \hyperref[eqn: forall i exists a in Ji such that prob of making mistake is bouded below using KL]{Equation \ref{eqn: forall i exists a in Ji such that prob of making mistake is bouded below using KL}}. Towards this, we will state a proposition.
    
    \begin{propn}
    \label{propn: min-max for s and alpha s}
    There exists a state $s\in [k]$ such that $$\sqrt{\frac{m_s}{18Tr_s}}\geq \sqrt{\frac{\sum_{\ell\in[k] } m_\ell }{18T}}$$
    \end{propn}
    \begin{proof}
    First, we note the following claim considering all vectors $\rho=\{\rho_1,\dots,\rho_k\}$ in the probability simplex $\Delta$.
    \begin{claim}
    \label{claim: inequality for sum mi-s}
    For any given set of integers $m_1, m_2, \ldots, m_k$, we have 
    \begin{align*}
        \min_{(\rho_1, \rho_2, \ldots, \rho_k) \in\Delta} \ \left( \max_{\ell\in[k]} \frac{m_\ell}{\rho_\ell} \right) \geq \sum_{\ell\in[k]} m_\ell
    \end{align*}
    \end{claim}
    \begin{proof}
        Assume, towards a contradiction, that for all $\ell\in [k]$, we have $\frac{m_\ell}{\rho_\ell} < \sum_{\ell\in[k]} m_\ell$. Then, $\rho_\ell > \frac{m_\ell}{\sum_{\ell\in[k]} m_\ell}$, for all $\ell \in [k]$. Therefore, $\sum_{\ell\in[k]} \rho_\ell > \sum_{\ell\in[k]}\frac{m_\ell}{\sum_{\ell\in[k]} m_\ell} = 1$. However, this is a contradiction as $\sum_{\ell\in[k]} \rho_\ell = 1$.
    \end{proof}
    
    An immediate consequence of \hyperref[claim: inequality for sum mi-s]{Claim \ref{claim: inequality for sum mi-s}} is that $$\min_{(r_1, r_2, \ldots, r_k) \in\Delta} \left( \max_{\ell\in[k]}  \sqrt{\frac{m_\ell}{18Tr_\ell}}\right) \geq \sqrt{\frac{\sum_{\ell \in [k]} m_\ell}{18T}}$$.

    Therefore, irrespective of how $r_i$s are chosen, there always exists a state $s\in[k]$ such that $\sqrt{\frac{m_s}{18Tr_s}}\geq \sqrt{\frac{\sum_{\ell \in [k]} m_\ell}{18T}}$. 
    \end{proof}
    
    For such a state $s\in[k]$ that satisfies \hyperref[propn: min-max for s and alpha s]{Proposition \ref{propn: min-max for s and alpha s}}, we note that, $\frac{m_s}{18Tr_s}\geq \beta^2$ or $\frac{18Tr_s\beta^2}{m_s}\leq 1$.
    
    Let us now restate \hyperref[eqn: forall i exists a in Ji such that prob of making mistake is bouded below using KL]{Equation \ref{eqn: forall i exists a in Ji such that prob of making mistake is bouded below using KL}} for such a state $s$.
    There exists a state $s \in [k]$ and an intervention $\alpha_s$ that satisfies 
    \begin{align}
    \label{eqn: exists s exists a-s in Ji such that prob of making mistake is bouded below using KL}
      \prob_{\P_{(\alpha_s,s)}}\{G(\alpha_s, s) \} \geq \frac{1}{2}\exp\left(-18\frac{Tr_s \beta^2}{m_s}\right) - \frac{1}{\lvert \J_s \rvert} \geq \frac{1}{2e} - \frac{1}{\lvert \J_s \rvert}
    \end{align}
    
    Note that the last inequality lower bounds the to probability of selecting a non-optimal policy when the algorithm $\mathscr{A}$ is executed on instance $\mathcal{F}_{\alpha_s, s}$.
    Furthermore, in instance $\mathcal{F}_{\alpha_s, s}$, for any non-optimal policy $\widehat{\pi}$ we have $\varepsilon(\widehat{\pi}) = \left( \frac{1}{2} + \beta \right) - \frac{1}{2} = \beta$. Therefore, we can lower bound $\mathscr{A}$'s regret over instance $\mathcal{F}_{\alpha_s, s}$ as follows: 
    \begin{align}
        \Regret_T &= \E[\varepsilon(\widehat{\pi})]\\
        &= \prob_{\P_{(\alpha_s,s)}}\{G (\alpha_s, s) \} \cdot \E[\Regret \mid G(\alpha_s, s)]\enspace +\enspace \prob_{\P_{(\alpha_s,s)}}\{G^c(\alpha_s, s)\} \cdot \E[\Regret \mid G^c(\alpha_s, s)]\nonumber\\
        &\geq \left[\frac{1}{2e} - \frac{1}{\lvert \J_s \rvert}\right] \beta +\enspace \prob_{\P_{(\alpha_s, s)}}\{G^c(\alpha_s, s)\}\cdot 0\nonumber\\
        &= \left[\frac{1}{2e} - \frac{1}{\lvert \J_s \rvert}\right] \beta
        \label{eqn: regret in terms of e Js and beta}
    \end{align}
    Note that we can construct the instances to ensure that $m_\ell \geq 8$, for all states $\ell$, and, hence, $\left(\frac{1}{2e} - \frac{1}{\lvert \J_i \rvert} \right) = \Omega(1)$ (see Proposition \ref{propn: Size of Ai}). Therefore \hyperref[eqn: regret in terms of e Js and beta]{Equation \ref{eqn: regret in terms of e Js and beta}} gives us:
    \begin{align}
        \Regret_T&= \Omega (\beta) = \Omega \left(\sqrt{\frac{\sum_{\ell\in[k]}m_\ell}{T}}\right) \label{ineq: regret case 1}
    \end{align}

    \textbf{Case 2} We now consider the case when $\beta = \frac{1}{3}$. In such a case, $\sqrt{\frac{\sum_{\ell\in[k] } m_\ell }{18T}}> \frac{1}{3}$. 
    
    We showed in \hyperref[propn: min-max for s and alpha s]{Proposition \ref{propn: min-max for s and alpha s}} that there exists a state $s\in [k]$ such that $\sqrt{\frac{m_s}{18Tr_s}}\geq \sqrt{\frac{\sum_{\ell\in[k] } m_\ell }{18T}}$. Combining the two statements, there exists a state $s$ such that $\sqrt{\frac{m_s}{18Tr_s}}\geq \frac{1}{3}$. 
    We now restate Inequality \ref{eqn: forall i exists a in Ji such that prob of making mistake is bouded below using KL} for such a state $s\in[k]$: $$\prob_{\P_{(\alpha_s,s)}}\{G(\alpha_s, s) \} \geq \frac{1}{2}\exp\left(-9 \beta^2\right) - \frac{1}{\lvert \J_s \rvert} = \frac{1}{2e} - \frac{1}{\lvert \J_s \rvert}$$
    
    Following the exact same procedure as in Case 1, we can derive that $\Regret_T \geq \left[\frac{1}{2e} - \frac{1}{\lvert \J_s \rvert}\right] \beta$. We saw in Case 1 that it is possible to construct instances such that $\left[\frac{1}{2e} - \frac{1}{\lvert \J_s \rvert}\right] = \Omega(1)$. Therefore the following holds for Case 2 also: 
    \begin{align}
        \Regret_T&= \Omega (\beta) = \Omega \left(\sqrt{\frac{\sum_{\ell\in[k]}m_\ell}{T}}\right) \label{ineq: regret case 2}
    \end{align}
    
    \comment{
    Therefore we have in both cases that there exists a state $s\in [k]$ and an intervention $\alpha_s \in \I_s$ such that 
    $$\prob_{\P_{(\alpha_s,s)}}\{G(\alpha_s, s) \} \geq \frac{1}{2e} - \frac{1}{\lvert \J_s \rvert}$$.

    \begin{align}
        \Regret_T = \E[\varepsilon(\widehat{\pi})]&= \prob_{\P_{(\alpha_s,s)}}\{G (\alpha_s, s) \} \cdot \E[\Regret \mid G(\alpha_s, s)]\enspace +\enspace \prob_{\P_{(\alpha_s,s)}}\{G^c(\alpha_s, s)\} \cdot \E[\Regret \mid G^c(\alpha_s, s)]\nonumber\\
        &\geq \left[\frac{1}{2e} - \frac{1}{\lvert \J_s \rvert}\right] \beta +\enspace \prob_{\P_{(\alpha_s, s)}}\{G^c(\alpha_s, s)\}\cdot 0\nonumber\\
        &= \left[\frac{1}{2e} - \frac{1}{\lvert \J_s \rvert}\right] \beta
        \label{eqn: regret case 2in terms of e Js and beta}
    \end{align}

    \begin{align}
        \Regret_T = \left(\frac{1}{2e} - \frac{1}{\lvert \J_\ell \rvert}\right)\beta_\ell= \Omega\left(\beta_\ell\right) = \Omega\left(\sqrt{\frac{m_\ell}{18Tr_\ell}}\right)
    \end{align}
    Therefore, in both the cases, $\Regret_T= \Omega\left(\sqrt{\frac{m_\ell}{18Tr_\ell}}\right)$

Amortizing over the states, we will show that there exists a state $s$ for which this lower bound is sufficiently large. Towards this, we establish the following proposition. 

Note that, by definition, $r_i$s constitute a vector in the standard simplex $\Delta$, i.e., $\sum_{i=1}^k r_i = 1$ and $r_i \in [0,1]$ for all $i \in [k]$. We now note the following proposition considering all vectors in the simplex $\Delta$. 
    \begin{propn}
    \label{propn: inequality for sum mi-s}
    For any given set of integers $m_1, m_2, \ldots, m_k$, we have 
    \begin{align*}
        \min_{(\rho_1, \rho_2, \ldots, \rho_k) \in\Delta} \ \left( \max_{i\in[k]} \frac{m_i}{\rho_i} \right) \geq \sum_{i\in[k]} m_i
    \end{align*}
    \end{propn}
    \begin{proof}
        Assume, towards a contradiction, that for all $i\in [k]$, we have $\frac{m_i}{\rho_i} < \sum_{i\in[k]} m_i$. Then, $\rho_i > \frac{m_i}{\sum_{i\in[k]} m_i}$, for all $i \in [k]$. Therefore, $\sum_{i\in[k]} \rho_i > \sum_{i\in[k]}\frac{m_i}{\sum_{i\in[k]} m_i} = 1$. However, this is a contradiction as $\sum_{i\in[k]} \rho_i = 1$.
    \end{proof}
    
    Therefore, \hyperref[eqn: Rt forall i exists a in lower bound terms]{Equation \ref{eqn: Rt forall i exists a in lower bound terms}} and \hyperref[propn: inequality for sum mi-s]{Proposition \ref{propn: inequality for sum mi-s}} 
    }

    Inequalities \ref{ineq: regret case 1} and \ref{ineq: regret case 2} imply that there exists a state $s$ and an intervention $\alpha_{s}$ such that, under instance  $\mathcal{F}_{( \alpha_s,s)}$, algorithm $\mathscr{A}$'s regret satisfies  
    \begin{align}
    \label{eqn: Regret in terms of sum mi}
        \Regret_T = \Omega \left(\sqrt{\frac{\sum_{\ell\in[k]}m_\ell}{T}}\right)
    \end{align}
    
    We complete the proof of Theorem \ref{theorem: Lower Bound for our algorithm} by showing that in the current context $\lambda = \sum_{\ell\in[k]} m_\ell$. 
    
    \begin{propn}
    \label{propn: Value of Lambda for chosen transition probability matrix}
    For the chosen transition matrix  
    $$\lambda := \min_{\text{fq.~vector} f} \ \left\lVert PM^{1/2}  \left(P^\tr  f \right)^{\circ-\frac{1}{2}} \right\rVert_\infty^2  = \sum_{\ell\in[k]}m_\ell$$
    \end{propn}
    \begin{proof}
    Recall that all the instances, $\mathcal{F}_0$ and $\mathcal{F}_{(a,i)}$s, have the same (deterministic) transition matrix $P$. Also, parameter $\lambda$  is computed via  \hyperref[eqn:lambda]{Equation \ref{eqn:lambda}}. 
    
    Consider any frequency vector $f$ over the interventions $\{1,\dots,N\}$. From the chosen transition matrix, we have the following:

    \begin{align*}
    P = 
    \begin{bmatrix}
    1 & 0 & \dots & 0 \\
    0 & 1 & \dots & 0 \\
    &&\dots\\
    0 & 0 & \dots & 1 \\
    &&\dots\\
    0 & 0 & \dots & 1 \\
    \end{bmatrix}& \qquad
    PM^{\frac{1}{2}} = 
    \begin{bmatrix}
    \sqrt{m_1} & 0 & \dots & 0 \\
    0 & \sqrt{m_2} & \dots & 0 \\
    &&\dots\\
    0 & 0 & \dots & \sqrt{m_k} \\
    &&\dots\\
    0 & 0 & \dots & \sqrt{m_k} \\
    \end{bmatrix}& \qquad
    P^\tr f = 
    \begin{bmatrix}
    f_1\\ f_2\\ \dots\\ f_{k-1}\\ f_k + \ldots + f_N\\
    \end{bmatrix}
    \end{align*}    
    
    From here, we can compute the following:
    \begin{align*}
      PM^{1/2}  \left(P^\tr  f \right)^{\circ-\frac{1}{2}} =
      \left[
      \sqrt{\frac{m_1}{f_1}},  \dots,  \sqrt{\frac{m_{k-1}}{f_{k-1}}},  \sqrt{\frac{m_k}{f_k + \ldots + f_N}},  \dots,  \sqrt{\frac{m_k}{f_k + \ldots + f_N}}
      \right]^\tr
    \end{align*}
    
    That is, for all $\ell \in [k-1]$, the $\ell$th component of the vector $PM^{1/2}  \left(P^\tr  f \right)^{\circ-\frac{1}{2}}$ is equal to $\sqrt{\frac{m_i}{f_i}}$. All the remaining components are $\sqrt{\frac{m_k}{f_k + \ldots + f_N}}$. 
    
    Write $\rho_\ell := f_\ell$ for all $\ell \in [k-1]$ and $\rho_k = \sum_{j=k}^N f_j$. Since $f$ is a frequency vector, $(\rho_1, \ldots \rho_k) \in \Delta$. In addition,  
    $$PM^{1/2}  \left(P^\tr  f \right)^{\circ-\frac{1}{2}} =\left[\sqrt{\frac{m_1}{\rho_1}},\dots ,\sqrt{\frac{m_{k-1}}{\rho_{k-1}}},\sqrt{\frac{m_k}{\rho_k}},\dots,\sqrt{\frac{m_k}{\rho_k}}\right]^\tr$$
    Therefore, by definition, $\lambda =\min_{(\rho_1, \ldots, \rho_k) \in\Delta} \left( \max_{\ell\in[k]} {\frac{m_\ell}{\rho_\ell}} \right)$. Now, using a complementary form of \hyperref[claim: inequality for sum mi-s]{Claim \ref{claim: inequality for sum mi-s}} we obtain $\lambda = {\sum_{\ell\in[k]}m_\ell}$. The proposition stands proved. 
    \end{proof}
    
    Finally, substituting \hyperref[propn: Value of Lambda for chosen transition probability matrix]{Proposition \ref{propn: Value of Lambda for chosen transition probability matrix}} into \hyperref[eqn: Regret in terms of sum mi]{Equation \ref{eqn: Regret in terms of sum mi}}, we obtain that there exists an instance $\mathcal{F}_{(\alpha_s, s)}$ for which algorithm $\mathscr{A}$'s regret is lower bounded as follows 
    \begin{align}
        \Regret_T = \Omega \left(\sqrt{\frac{\lambda}{T}}\right).
    \end{align}
    This completes the proof of Theorem \ref{theorem: Lower Bound for our algorithm}.

    \section{Experiments}
	\label{section: experiments}
	\vspace{-0.08in}
	We first describe \hyperref[section: UE Description]{\UE} (Uniform Exploration Algorithm), the baseline algorithm that we compare  \hyperref[alg:best policy generator]{\CE} with. This is followed by a complete description of our experimental setup. Finally, we present and discuss our main results. 
	
	\label{section: UE Description}
	\textbf{Uniform Exploration (\hyperref[section: UE Description]{\UE}):} This algorithm uniformly explores all the interventions in the instance. It first performs all the interventions $a \in \I_0$ at the start state $0$ in a round robin manner. On transitioning to any state $i\in [k]$, it performs interventions $b\in\I_i$ in a round robin manner. 

	\textbf{Setup:} We consider an MDP with a start state $0$, $k=25$ intermediate states and a terminal state. 
	At each state we have a causal graph with $n=25$ variables. 
	The number of interventions is therefore $N=2n+1 = 51$. Our reward is a Bernoulli random variable, with probability $0.5+\varepsilon$, if $X^1_1 = 1$ and $0.5$ for every other intervention (we use $\varepsilon=0.3$ in our experiments). Note that the reward function is unknown to the algorithm. Like in Lattimore et al.~\cite{Lattimore}, we set $q^i_j=0$ for $j\leq m_i$ and $0.5$ otherwise. In our setup, we set all $m_i$ values for all intermediate states $i\in[k]$ to be the same. 
	On taking action $a = do()$ at state $0$, we transition uniformly to one of the intermediate states. On taking action $do(X^0_i=1)$ (at state $0$), where $i\in [n]$, we transition with probability $\frac{2}{k}$ to state $i$ (we take $k=n$ for these experiments) and probability $\frac{1}{k} - \frac{1}{k(k-1)}$ to any other state. Recall that $q^0_i = \prob(X^0_i=1)$. Then, for all interventions $do(X^0_i=1)$, we have a transition probability vector given by $p_{(X^0_i=1)}$ which is computed from $(1-q^0_i)p_{(X^0_i=0)} + q^0_i p_{(X^0_i=1)} = p_{do()}$.

	We perform two experiments on the above model. In the first one, we run \hyperref[alg:best policy generator]{\CE} and \hyperref[section: UE Description]{\UE} for  time horizon $T\in\{1000,\ldots, 25000\}$. In the second experiment, we run \hyperref[alg:best policy generator]{\CE} and \hyperref[section: UE Description]{\UE} for a fixed time horizon $T=25000$ with $\lambda$ varying in the set $\{50, 75, \ldots, 625\}$. To vary $\lambda$, we vary $m_i$ for the intermediate states in the set $\{2,3,\ldots,25\}$. In both the experiments we average the regret over $10000$ independent runs for each setting. We use CVXPY \cite{CVXOPT} to solve the optimization problem at \hyperref[step:fstar]{Step \ref{step:fstar}} in \hyperref[alg:best policy generator]{\CE}.

    \begin{figure}[t]
    \centering
    \begin{subfigure}[t]{0.48\textwidth}
        \centering
        \includegraphics[width=\textwidth]{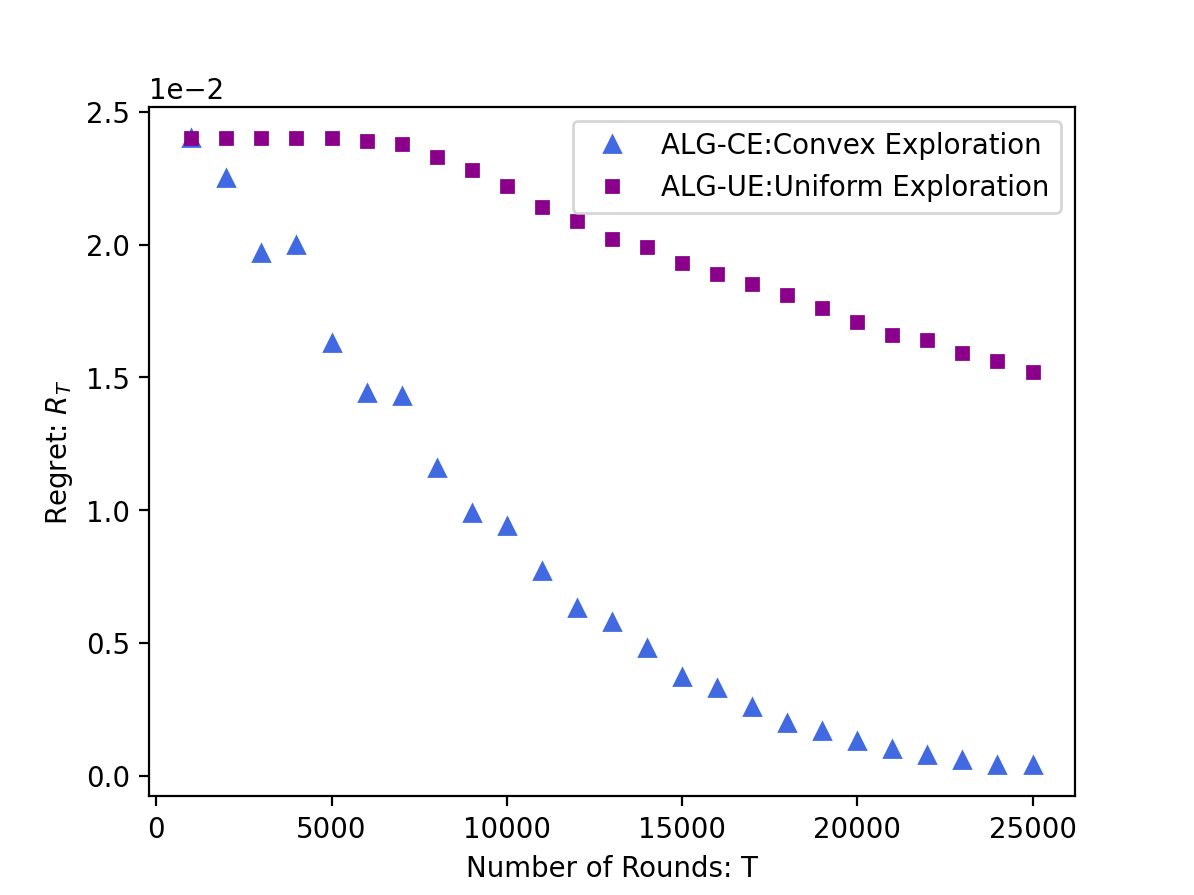}
        \caption{Expected Regret vs Time, for setup with $n=25,\ k=25,\ \lambda=50,\ \varepsilon = 0.3$ and $m=2$ for all states.}
        \label{fig:Regret with T}
    \end{subfigure}
    \hfill
    \begin{subfigure}[t]{0.48\textwidth}
        \centering
        \includegraphics[width=\textwidth]{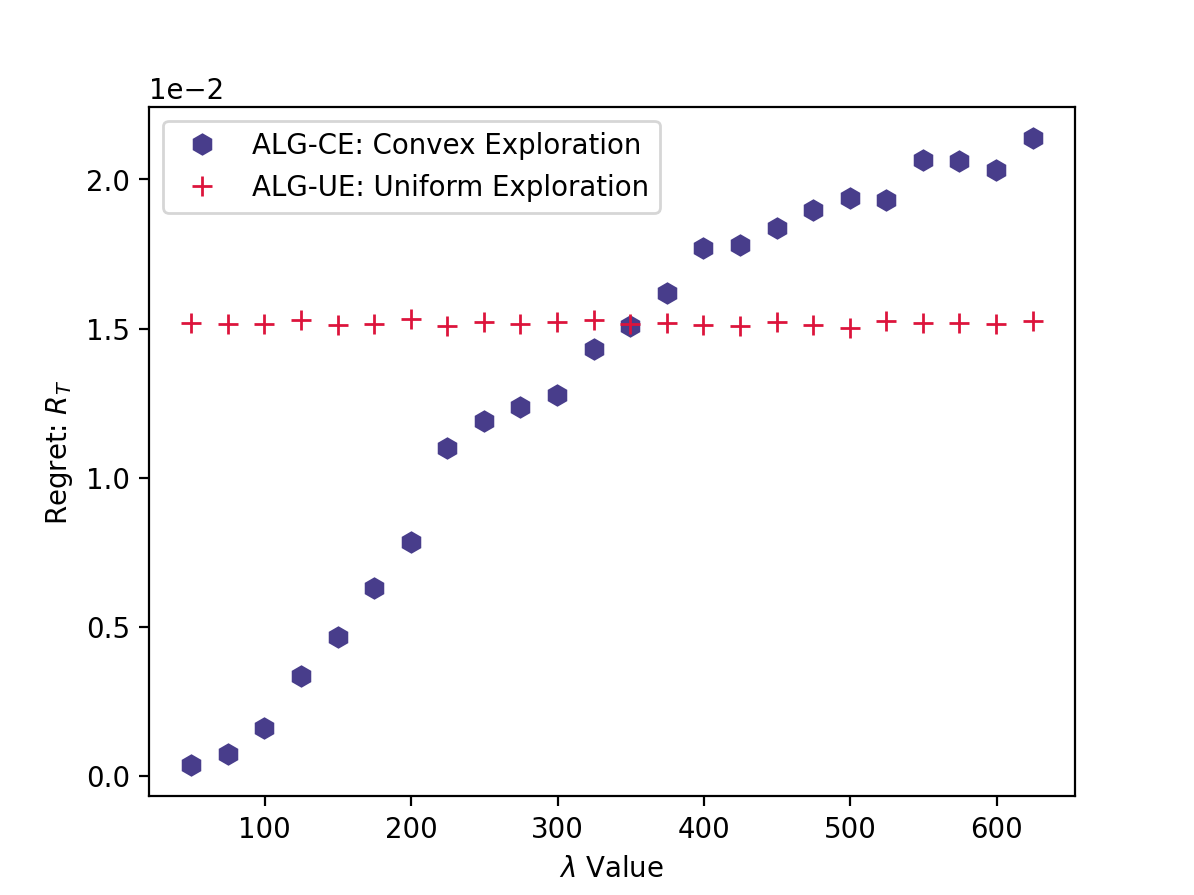}
        \caption{Expected Regret vs $\lambda$, for fixed $T=25000$, $n = 25,\ k=25,\ m_0 = 2$ and $\varepsilon = 0.3$.}
        \label{fig:Regret with Lambda}
    \end{subfigure}
    \caption{Experimental Results}
    \label{fig:Experimental Results}
    \end{figure}

	\textbf{Results:} In Figure \ref{fig:Regret with T}, we compare the expected simple regret of \hyperref[alg:best policy generator]{\CE} and  \hyperref[section: UE Description]{\UE} obtained in the first experiment. Our plots indicate that \hyperref[alg:best policy generator]{\CE} outperforms \hyperref[section: UE Description]{\UE} and its regret falls rapidly as $T$ increases.  In Figure \ref{fig:Regret with Lambda}, we plot the expected simple regret against $\lambda$ for \hyperref[alg:best policy generator]{\CE} and \hyperref[section: UE Description]{\UE} that was obtained in Experiment $2$, and empirically validate their relationship that was proved in Theorem \ref{theorem:main}. 
	
	Figure \ref{fig:Regret with Lambda} shows that \hyperref[alg:best policy generator]{\CE} outperforms \hyperref[section: UE Description]{\UE} for a wide range of $\lambda$s. This  highlights the applicability of \hyperref[alg:best policy generator]{\CE}, specifically in instances wherein $\lambda$ is dominated by the other instance parameters. This also substantiates the relevance of using causal information; note that, by construction, \hyperref[alg:best policy generator]{\CE} uses such information, whereas \hyperref[section: UE Description]{\UE} does not.

	\section{Conclusion and Future Work}
	We studied extensions of the causal bandits framework into causal MDPs by considering a two-stage setup. This accounted for non-trivial extensions from \cite{Lattimore} by considering multiple states as well as transitions between states. 
	We developed the Convex Exploration algorithm for minimizing simple regret in this model. We also identified an instance dependent parameter $\lambda$, and proved that the regret of this algorithm is $\tilde{\O}\left(\sqrt{\frac{1}{T}\max\{\lambda,\frac{m_0}{p_+}\}}\right)$. The current work also established that, for certain families of instances, this upper bound is essentially tight. Finally, we showed through experiments that our algorithm performs better than uniform (naive) exploration in a range of settings.
	
	Instead of addressing MDPs, in their unruly generality, this work considers a stepped extension of causal bandits. Building upon this, a natural way forward would be to consider multi-stage MDPs. This would likely entail the development of an exploration algorithm that solves optimization problems at various stages of the MDP. Furthermore, while our work relies on the parallel causal graph model of Lattimore et al.~\cite{Lattimore}, an active area of research is to address general causal graphs for causal bandits. Extending causal MDPs to such contexts is an interesting direction of future work.

\bibliography{References.bib}
\clearpage

\appendix
\appendixpage
\renewcommand{\baselinestretch}{1.1}

\section{Bounding the Probability of Bad Event}
	\label{appendixsection: Bounding bad events}
	
	Recall that the \texttt{good event} corresponds to $\bigcap_{i\in5}  \event i $ (see Definition \ref{defn: good events}). Write $F := \neg\left(\bigcap_{i\in5}  \event i \right)$ and note that, for the regret analysis, we require an upper bound on $\prob\{F\} = \prob\left\{\neg(\bigcap_{i\in5}  \event i )\right\} = \prob\left\{\bigcup_{i\in5} \neg  \event i \right\}$. Towards this, in this section we address $\prob\{\neg  \event i \}$, for each of the events $ \event 1 $-$ \event 5$, and then apply the union bound.

	\subsection{$ \event 1 $: Probability Bound}
	\label{appendixsection: Bounding neg E1}
	
	The next lemma upper bounds the probability of $\neg  \event 1 $.
	\begin{lem}
		\label{lem: Bound on E1}
		In each of Algorithms \ref{alg:estimateTransitionProbabilities}, \ref{alg:estimateCausalParameters} and \ref{alg:estimateRewards} and for all interventions $a \in \I_0$, we have 
		\begin{align*}
		    \prob\{\neg  \event 1 \} = \prob\left\{\sum\limits_{i=1}^k \lvert \widehat{P}_{(a,i)} - P_{(a,i)} \rvert >\frac{ p_+ }{3}\right\}<\frac{k}{T}
		\end{align*}
		whenever $T\geq \max\left\{\frac{1620 N }{ p_+ ^3},\frac{2025 N }{ p_+ ^2} \log\left(\frac{9NT}{k}\right)\right\}$.
	\end{lem}
	\begin{proof}
		On performing any intervention $a\in\I_0$ at state $0$, the intermediate state that we visit follows a multinomial distribution. 	Hence, we can apply Devroye's inequality (for multinomial distributions) to obtain a concentration guarantee; we state the inequality next in our notation.

	\begin{lem}[{Restatement of Lemma 3 in \cite{Devroye1983}}]
	\label{eqn: equivalent condition to luc devroye}
    Let $T_a$ be the number of times intervention $a\in\I_0$ is performed in state $0$. Then, for any $\eta >0$ and any $T_a\geq \frac{20s}{\eta^2}$, we have 
    \begin{align*}
        \prob\left\{\sum\limits_{i=1}^k \lvert \widehat{P}_{(a,i)} - P_{(a,i)} \rvert > \eta\right\}\leq 3\exp\left(-\frac{T_a\eta^2}{25}\right)
    \end{align*}
    Here, $s$ is the support of the distribution (i.e., the number of states that can be reached from $a$ with a nonzero probability).
	\end{lem}

		Note that each intervention $a\in\I_0$ is performed at least $T_a = \frac{T}{9N}$ times across Algorithms \ref{alg:estimateTransitionProbabilities}, \ref{alg:estimateCausalParameters} and \ref{alg:estimateRewards}. 
		Setting $\eta = \frac{ p_+ }{3}$ and $T_a = \frac{T}{9N}$ above, we get that for each intervention $a\in\I_0$, in each subroutine:
		\begin{align*}
		    \prob\left\{\sum_{i=1}^k \lvert P_{(a,i)} - \widehat{P}_{(a,i)} \rvert > \frac{ p_+ }{3}\right\} \leq 3\exp\left(-\frac{T p_+ ^2}{9N \cdot  9\cdot 25}\right)=3\exp\left(-\frac{T p_+ ^2}{2025 N }\right)
		\end{align*}
		
		Note that to apply the inequality, we require $\frac{T}{9N}\geq \frac{180s}{ p_+ ^2}$, i.e., $T\geq \frac{1620 sN}{ p_+ ^2}$. In the current context, the support size $s$ is at most $\frac{1}{p_+}$; this follows from the fact that on performing any intervention $a\in\I_0$, at most $\frac{1}{p_+}$ states can have $P_{(a,i)}\geq p_+$. Hence, the requirement reduces to $T\geq \frac{1620 N}{ p_+ ^3}$.
		
		Next, we union bound the probability over the $N$ interventions (at state $0$) and the three subroutines, to obtain that, for any intervention $a \in \I_0$ and in any subroutine,  
		\begin{align*}
		    \prob\left\{\sum_{i=1}^k \lvert P_{(a,i)} - \widehat{P}_{(a,i)} \rvert > \frac{ p_+ }{3}\right\} \leq 3N \cdot 3\exp\left(-\frac{T p_+ ^2}{2025N}\right) = 9N\exp\left(-\frac{T p_+ ^2}{2025N}\right)
		\end{align*}

		Note that $9N\exp\left(-\frac{T p_+ ^2}{2025N}\right) \leq \frac{k}{T}$, for any $T \geq \frac{2025N}{ p_+ ^2} \log\left(\frac{9NT}{k}\right)$. 
		Hence, for any $T\geq \max\left\{\frac{1620 N }{ p_+ ^3}\frac{2025 N }{ p_+ ^2} \log\left(\frac{9NT}{k}\right)\right\}$, we have:
		\begin{align*}
		    \prob[\neg  \event 1 ] \leq 9N\exp\left(-\frac{T p_+ ^2}{2025N}\right) \leq \frac{k}{T}
		\end{align*}
		This completes the proof of the lemma. 
	\end{proof}

	\subsection{$ \event 2 $: Probability Bound}
	\label{appendixsection: Bounding neg E2 and neg E3}
	
	In this section, we bound the probabilities that our estimated $\widehat{m}_i$s are far away from the true causal parameters $m_i$s. 
	
	\begin{lem}
		\label{lem: Bound on E2}
		For any $T \geq 144m_0 \log\left(\frac{TN}{k}\right)$, in \hyperref[alg:estimateTransitionProbabilities]{Algorithm \ref{alg:estimateTransitionProbabilities}}, 
		\begin{align*}
		   \prob[\neg  \event 2 ] = \prob\left\{\widehat{m}_0 \notin [\frac{2}{3}m_0,2m_0] \right\}\leq \frac{k}{T} 
		\end{align*}
	\end{lem}
	
	\begin{proof}
		We allocate time $\frac{T}{3}$ to \hyperref[alg:estimateTransitionProbabilities]{Algorithm \ref{alg:estimateTransitionProbabilities}}. Lemma 8 in \cite{Lattimore} ensures that, for any $\delta \in (0,1)$ and $\frac{T}{3} \geq 48 m_0 \log(\frac{N}{\delta})$, we have $\widehat{m}_0 \in [\frac{2}{3}m_0,2m_0]$, with probability at least $(1-\delta)$.    
		Setting $\delta = \frac{k}{T}$, we get the required probability bound.
	\end{proof}
	
	\subsection{$ \event 3 $: Probability Bound}
	Next, we address $\prob\{\neg  \event 3   \mid   \event 1 \}$. 
	\begin{lem}
		\label{lem: Bound on E3}
		For any $T\geq \frac{648\max(m_i)N}{p_+} \log\left(2NT\right)$, in each of Algorithms \ref{alg:estimateCausalParameters} and \ref{alg:estimateRewards}, we have
		\begin{align*}
		    \prob\left\{\exists i\in[k],\quad \widehat{m}_i \notin [\frac{2}{3}m_i,2m_i] \ \big \vert  \  \event 1  \right\}\leq \frac{k}{T}
		\end{align*}
	\end{lem}
	\begin{proof}
		Fix any reachable state $i\in[k]$. Corresponding to such a state, there exists an intervention $\alpha\in\I_0$ such that $P_{(\alpha,i)} \geq p_+$. Event $ \event 1 $ (Corollary \ref{cor: corollary to E1 as a multiplicative bound}) implies that $\widehat{P}_{(\alpha,i)} \geq \frac{2}{3} P_{(\alpha,i)} \geq \frac{2}{3}p_+$.
		
		Now, write $T_i$ to denote the number of times state $i\in[k]$ is visited by the Algorithms \ref{alg:estimateCausalParameters} and \ref{alg:estimateRewards}. Recall that in the subroutines we estimate $\widehat{P}_{(\alpha,i)}$ by counting the number of times state $i$ was reached and simultaneously intervention $\alpha$ observed. Furthermore, note that we allocate to every intervention at least $\frac{T}{9N}$ time (See Steps 2 in both the subroutines). In particular, intervention $\alpha$ was necessarily observed $\frac{T}{9N}$ times. Therefore, $\widehat{P}_{(a,i)} \leq \frac{T_i}{\left(\frac{T}{9N}\right)}$. This inequality leads to a useful lower bound: $T_i \geq \frac{T}{9N} \ P_{(a,i)} \geq T\frac{2p_+}{27N}$.

		We now restate Lemma 8 from \cite{Lattimore}:
		Let $T_i$ be the number of times state $i\in[k]$ is observed. Then,
		\begin{align*}
		    \prob\left\{\widehat{m}_i\notin[\frac{2}{3}m_i,2m_i]\right\} \leq 2N\exp\left(-\frac{T_i}{48m_i}\right)
		\end{align*}

		Since $T_i \geq \frac{2Tp_+}{27N}$, this guarantee from \cite{Lattimore} corresponds to
		\begin{align*}
		    \prob\left\{\widehat{m}_i\notin[\frac{2}{3}m_i,2m_i]\ \big \vert  \  \event 1 \right\} \leq 2N\exp\left(-\frac{Tp_+}{648N m_i}\right) \leq 2N\exp\left(-\frac{Tp_+}{648N\max(m_i)}\right)
		\end{align*} 
		
		Union bounding over all states $i\in[k]$ and the two Algorithms \ref{alg:estimateCausalParameters} and \ref{alg:estimateRewards}, we obtain: 
		\begin{align*}
		    \prob\left\{\exists i \in[k] \text{ in Algorithms \ref{alg:estimateCausalParameters}, \ref{alg:estimateRewards} } \text{ with } \widehat{m}_i \notin [\frac{2}{3}m_i,2m_i] \ \big \vert  \  \event 1  \right\} \leq 2Nk\exp\left(-\frac{Tp_+}{648N\max(m_i)}\right)
		\end{align*}
		Finally, substituting the value of $T\geq \frac{648\max(m_i)N}{p_+} \log\left(2NT\right)$, gives us:
		\begin{align*}
		    \prob\{\exists i \in[k] \text{ in Algorithms \ref{alg:estimateCausalParameters}, \ref{alg:estimateRewards} } \text{ with } \widehat{m}_i &\notin [\frac{2}{3}m_i,2m_i] \ \big \vert  \  \event 1  \}\\
		    &\leq 2Nk\exp\left(-\frac{p_+}{648N\max(m_i)}\cdot \left[\frac{648\max(m_i)N}{p_+} \log\left(2NT\right)\right]\right)\\
		    &=\frac{k}{T}
		\end{align*}
		This completes the proof.
	\end{proof}

	\subsection{$ \event 4$: Probability Bound}
	\label{appendixsection: Bounding neg E4}
	
	The following lemma provides an upper bound for $\prob\{\neg  \event 4 \mid   \event 2 \}$.
	\begin{lem}
		\label{lem: Bound on E4}
		Let $\eta' := \sqrt{\frac{150m_0}{Tp_+}\log\left(\frac{3T}{k}\right)}$. Then,
		\begin{align*}
		    \prob\{\neg  \event 4 \mid   \event 2 \} =  \prob\left\{\sum\limits_{i\in[k]}\left\lvert P_{(a,i)}-\widehat{P}_{(a,i)}\right\rvert > \eta' \big \vert   \event 2  \right\} \leq \frac{k}{T}
		\end{align*}
	\end{lem}
	\begin{proof}
		As in the proof of Lemma \ref{lem: Bound on E1}, we will use Devroye's inequality. Write $T_a$ to denote the number of times intervention $a\in\I_0$ is observed (in state $0$) in  Algorithm \ref{alg:estimateTransitionProbabilities}. 
		
		For any $\eta \in (0,1)$ and with $T_a\geq \frac{20s}{ \eta ^2}$, Devroye's inequality gives us 
		\begin{align*}
		    \prob\left\{\sum\limits_{i=1}^k \lvert \widehat{P}_{(a,i)} - P_{(a,i)} \rvert >  \eta \right\}\leq 3\exp\left(-\frac{T_a \eta ^2}{25}\right)
		\end{align*}
		Here, $s$ is the size of the support of the multinomial distribution.
		
		We first show that $T_a$ is sufficiently large, for each intervention $a \in \I_0$. Recall that we allocate time $\frac{T}{3}$ to \hyperref[alg:estimateTransitionProbabilities]{Algorithm \ref{alg:estimateTransitionProbabilities}}. Furthermore, we observe each intervention in state $0$, at least $\frac{T}{3\widehat{m}_0}$ times, either as part of the do-nothing intervention or explicitly in Step \ref{step:alg2step10} of Algorithm \ref{alg:estimateTransitionProbabilities}. 
		Now, event $ \event 2 $ ensures that $\widehat{m}_0 \in [\frac{2}{3}m_0,2m_0]$. Hence, each intervention $a\in\I_0$ is observed $T_a \geq \frac{T}{3\widehat{m}_0} \geq  \frac{T}{3\cdot2m_0} = \frac{T}{6m_0}$ times.
		
		Substituting this inequality for $T_a$ in the above-mentioned probability bound, when $T\geq \frac{120sm_0}{\eta^2}$, we obtain:
		\begin{align*}
		    \prob\left\{\sum\limits_{i=1}^k \lvert \widehat{P}_{(a,i)} - P_{(a,i)} \rvert > \eta\right\}\leq 3\exp\left(-\frac{T\eta^2}{150m_0}\right)
		\end{align*}
		As observed in Lemma \ref{lem: Bound on E1}, the support size $s$ is at most $\frac{1}{p_+}$. Therefore, the requirement on $T$ reduces to $T\geq \frac{120m_0}{\eta^2 p_+}$. 		
		
		Setting $\eta = \sqrt{\frac{150m_0}{Tp_+}\log\left(\frac{3T}{k}\right)}$ gives us   \begin{align*}\prob\left\{\sum\limits_{i=1}^k \lvert \widehat{P}_{(a,i)} - P_{(a,i)} \rvert > \sqrt{\frac{150m_0}{Tp_+}\log\left(\frac{3T}{k}\right)}\right\} \leq 3\exp\left(-\frac{T}{150m_0}\left[\sqrt{\frac{150m_0}{Tp_+}\log\left(\frac{3T}{k}\right)}\right]^2\right) \leq \frac{k}{T}.
		\end{align*} 
		Therefore 
		\begin{align*}
		    \prob\left\{\sum\limits_{i=1}^k \lvert \widehat{P}_{(a,i)} - P_{(a,i)} \rvert > \eta\right\}\leq \frac{k}{T}
		\end{align*}
		This probability bound requires $T\geq \frac{120m_0}{\eta^2 p_+}$. That is, $\eta \geq \sqrt{\frac{120m_0}{Tp_+}}$. This inequality is satisfied by our choice of $\eta$. Hence, the lemma stands proved. 
	\end{proof}
	
	\subsection{$ \event 5$: Probability Bound}
	\label{appendixsection: Bounding neg E5}
	The next lemma bounds $\prob\{\neg  \event 5 \mid  \event 1 , \event 3 \}$.
	\begin{lem}
		\label{lem: Bound on E5}
		Let $\widehat{\eta}_i = \sqrt{\frac{27 \widehat{m}_i}{T(\widehat{P}^\tr \widehat{f}^*)_i}\log\left(2TN \right)}$. Then, 
		\begin{align*}
		    \prob\left\{\exists i\in[k] \text{ and } a \in \I_i \text{ such that } \left \lvert  \E\left[R_i  \mid  a\right]-\widehat{\mathcal{R}}_{(a,i)} \right \rvert > \widehat{\eta}_i \ \mid \  \event 3 ,  \event 1 \right\} \leq \frac{k}{T}
		\end{align*}
		Equivalently, 
		\begin{align*}
		    \prob\{\neg  \event 5  \mid   \event 3 , \event 1 \} \leq \frac{k}{T}
		\end{align*}
	\end{lem}
	\begin{proof}
		For intermediate states $i \in [k]$, we denote the realization of the causal parameters $m_i$ and the transition probabilities $P$ in \hyperref[alg:estimateRewards]{Algorithm \ref{alg:estimateRewards}}, as $\widetilde{m}_i$ and $\widetilde{P}$, respectively.	
		The estimates in the previous subroutines are denoted by $\widehat{m}_i$ and $\widehat{P}$.

		Event $ \event 1 $ gives us $P_{(a,i)} \in [\frac{3}{4}\widehat{P}_{(a,i)},\frac{3}{2}\widehat{P}_{(a,i)}]$and $\widetilde{P}_{(a,i)} \in [\frac{2}{3}P_{(a,i)},\frac{4}{3} P_{(a,i)}]$. Hence, the estimates across the subroutines are close enough: $\widetilde{P}_{(a,i)} \in [\frac{1}{2}\widehat{P}_{(a,i)},2\widehat{P}_{(a,i)}]$. Similarly, event $ \event 3 $ gives us $\widetilde{m}_i \in [\frac{1}{3}\widehat{m}_i,3\widehat{m}_i]$.
		
		Write $\widetilde{T}_i$ to denote the number of times state $i\in [k]$ was visited in \hyperref[alg:estimateRewards]{Algorithm \ref{alg:estimateRewards}}. For all states $i \in [k]$, we first establish a useful lower bound on $\widetilde{T}_i$, under events $ \event 1 $ and $ \event 3 $. The relevant observation here is that the estimate $\widetilde{P}_{(\alpha, i)}$ was computed in \hyperref[alg:estimateRewards]{Algorithm \ref{alg:estimateRewards}} by counting the number of times state $i$ was visited with intervention $\alpha \in \I_0$ (at state $0$). By construction, in \hyperref[alg:estimateRewards]{Algorithm \ref{alg:estimateRewards}} each intervention $\alpha \in \I_0$ was performed at least $\frac{\widehat{f}^*_\alpha}{3} \frac{T}{3}$ times. Furthermore, given that $\widetilde{P}_{(\alpha, i)}$ was computed via the visitation count, we get that state $i$ is visited with intervention $\alpha \in \I_0$ \emph{at least} $\widetilde{P}_{(\alpha, i)} \frac{T \widehat{f}^*_\alpha}{9}$ times. Therefore, 
		\begin{align*}
		    \widetilde{T}_i \geq \sum_{\alpha \in \I_0} \ \widetilde{P}_{(\alpha, i)} \frac{T \widehat{f}^*_\alpha}{9} = \frac{T}{9} (\widetilde{P}^\tr \widehat{f}^*)_i \geq \frac{T}{18} (\widehat{P}^\tr \widehat{f}^*)_i
		\end{align*}
		Here, the last inequality follows from the above-mentioned proximity between  $\widehat{P}$ and $\widetilde{P}$.  
		
		Now, note that, at each state $i \in [k]$, \hyperref[alg:estimateRewards]{Algorithm \ref{alg:estimateRewards}} (by construction) observes every intervention $a\in \I_i$ at least $\frac{\widetilde{T}_i}{\widetilde{m}_i}$ times. Write $\widetilde{T}_{(a,i)}$ to denote the number of times intervention $a \in \I_i$ is observed in this subroutine. Hence, 
		\begin{equation}
		    \label{equation: inequality for tildeT}
			\widetilde{T}_{(a,i)} \geq \frac{\widetilde{T}_i}{\widetilde{m}_i} \geq \frac{1}{\widetilde{m}_i} \frac{T}{18} (\widehat{P}^\tr \widehat{f}^*)_i \geq \frac{1}{3 \widehat{m}_i} \frac{T}{18} (\widehat{P}^\tr \widehat{f}^*)_i 
		\end{equation}

		For each state $i \in [k]$ and intervention $a \in \I_i$, define the event $\neg  \event 5(a,i)$ as 
		$\lvert \E\left[R_i  \mid  a\right]-\widehat{\mathcal{R}}_{(a,i)} \rvert >\widehat{\eta}_i$. Hoeffding's inequality gives us: 
		\begin{align*}
		    \prob\left\{ \neg  \event 5{(a,i)} \mid  \event 1 ,  \event 3  \right\}\leq 2\exp\left(-2 \widetilde{T}_{(a,i)}\widehat{\eta}_i^2\right) \leq 2\exp\left(-T \frac{(\widehat{P}^\tr \widehat{f}^*)_i \widehat{\eta}_i^2}{27 \widehat{m}_i}\right)
		\end{align*}
		The last inequality is obtained by substituting \hyperref[equation: inequality for tildeT]{Equation \ref{equation: inequality for tildeT}}. 
		
		Recall that $\widehat{\eta}_i = \sqrt{\frac{27 \widehat{m}_i}{T(\widehat{P}^\tr \widehat{f}^*)_i}\log\left(2TN \right)}$. Hence, the previous inequality corresponds to  
		\begin{align*}
		    \prob\left\{ \neg  \event 5{(a,i)} \mid  \event 1 ,  \event 3  \right\} \leq 2\exp\left(-T \frac{(\widehat{P}^\tr \widehat{f}^*)_i}{27 \widehat{m}_i}\cdot \left[\sqrt{\frac{27 \widehat{m}_i}{T(\widehat{P}^\tr \widehat{f}^*)_i}\log\left(2TN \right)}\right]^2\right) = \frac{1}{TN}
		\end{align*}

		Note that $\neg  \event 5 = \bigcup_{i\in[k]}\bigcup_{a\in\I_i}  \event 5{(a,i)}$. Taking a union bound over all states $i \in [k]$ and interventions $a \in \I_i$, we obtain 
		\begin{align*}
		    \prob\{\neg  \event 5 \mid  \event 1 , \event 3 \} \leq \frac{kN}{TN} = \frac{k}{T}
		\end{align*}
		This completes the proof. 
	\end{proof}
	\subsection{Bound on \texttt{bad event} (F):}
	\label{section:pr-bad-event}
	Here we restate and prove Lemma \ref{lem: Bound on F}. 
	
    \boundonBadEvent*
    
    \begin{proof}
        To summarize the statements of Lemmas \ref{lem: Bound on E1}, \ref{lem: Bound on E2}, \ref{lem: Bound on E3}, \ref{lem: Bound on E4} and \ref{lem: Bound on E5}, let $T\geq T_0$ where $T_0$ is given by:
		\begin{align*}
		    T_0 &=  \max\left\{\frac{1620 N }{ p_+ ^3},\frac{2025 N }{ p_+ ^2} \log\left(\frac{9NT}{k}\right),144m_0 \log\left(\frac{Tn}{k}\right),\frac{864\max(m_i)N}{p_+} \log\left(2nT\right) \right\}\\
		    &= \O\left(\frac{N\max(m_i)}{p_+^3}\log\left(2NT\right)\right)
		\end{align*}
		Then we obtain:
		\begin{align*}
		    \prob\{\textrm{F}\} &= \prob\left\{ \left[\bigcup_{i\in[5]} \neg  \event i \right] \right\} \\
		    &\leq \prob\{\neg  \event 1 \} + \prob\{\neg  \event 2 \} + \prob\{\neg  \event 3  \mid   \event 1  \} + \prob\{\neg  \event 4 \mid   \event 2  \} + \prob\{\neg  \event 5 \mid   \event 3 , \event 1  \} \\
		    &\leq  \frac{5k}{T}
		\end{align*}
		This completes the proof.
	\end{proof}

	\section{Convexity of the Optimization Problems}
	\label{section: nature of optimization problems}
	\begin{propn}
		\label{propn: f tilde problem is an LP}
		Let $ \tilde{f} = \argmax\limits_{\text{fq.~vector} f} \enspace \min\limits_{\text{states [k]}}  \widehat{P}^\tr  f $. Then, finding $\tilde{f}$ is an LP
	\end{propn}
	\begin{proof}
		We rewrite the above $\max\limits_{\text{fq.~vector} f}\quad \min\limits_{i\in[k]}(\cdot)$ as a simpler program:
		\begin{align*}
			\max_{ f } \quad & z\\
			\textrm{subject to} \quad & \widehat{P}^\tr_1 f \geq z \\
			&\dots\\
			& \widehat{P}^\tr_{N} f \geq z \\
			& f \cdot \1 = 1 \\
			& f \succeq 0\\
		\end{align*}
		Where $N = \lvert  \I_0 \rvert$. This is equivalent to the standard form of a linear program, and hence is an LP.
	\end{proof}
	\begin{lem}
		\label{lem:optimization problem is convex}
		$\min\limits_{\text{fq.~vector} f}\quad  \max\limits_{\text{interventions } \I_0 } \widehat{P}\hat{M}^{\frac{1}{2}}\left[\widehat{P}^\tr f \right]^{\circ-\frac{1}{2}}$ is a convex optimization problem
	\end{lem}
	\begin{proof}
		First we write the $\min$-$\max$ in terms of a single minimization. First let us use the shorthand $A:= \widehat{P}\hat{M}^{\frac{1}{2}}$ and $\{A_1,\dots,A_N\}$ (where $N:=\lvert \I_0 \rvert$) denote the rows of the matrix
		\begin{align}
		  \textbf{OPT}:
			\min_{ f } \quad & z \nonumber \\
			\textrm{subject to} \quad & A_1\cdot \left[\widehat{P}^\tr f \right]^{\circ-\frac{1}{2}} \leq z  \nonumber \\
			&\dots \nonumber \\
			& A_{N}\cdot \left[\widehat{P}^\tr f \right]^{\circ-\frac{1}{2}} \leq z \label{eqn: optimization problem}\\
			& f\cdot \1 = 1 \nonumber \\
			& f\succeq 0 \nonumber
		\end{align}
		\comment{
			\begin{propn}
				\label{propn: constraint equations are convex in components of f}
				The constraint equations of \hyperref[eqn: optimization problem]{OPT} are convex in the components of $f$
			\end{propn}
			\begin{proof}
				We write $f = \{f_1,\dots,f_N\}$  WLOG, we will vary the first component $f_1$. Fix $\{f_2,\dots,f_N\}$. Then note that $P^\tr f = \{\widehat{P}(*,i)^\tr f\}_{i\in[k]} = \{\widehat{P}(1,i) f_1+ c_i\}_{i\in[k]}$ for some constants $c_i,i\in[k]$.
				
				Consider one of the constraints of our problem, viz $A_1^\tr \left[\widehat{P}^\tr f \right]^{\circ-\frac{1}{2}} \leq z$. We can simplify this to get: $\sum_{i\in[k]} \frac{A_{1i}}{\sqrt{\widehat{P}(1,i) f_1+ c_i}}$. 
				
				Consider the expression: $\frac{A_{1i}}{\sqrt{\widehat{P}(1,i) f_1+ c_i}}$. This is of the form $f(x) = \frac{a}{\sqrt{bx+c}}$ for some constants $a,b,c\in\R_+$. But this is convex as $f''(x) \geq 0$.
				
				\comment{We note that if $f(x) = h(g(x))$, then if $h(\cdot)$ is convex and non-increasing and $g(\cdot)$ is concave, then $f(\cdot)$ is convex. (Eq 3.10 \cite{Boyd}). We have that $h(x) = \frac{1}{x}$ is convex and non-increasing, and that $g(x) = \sqrt{bx+c}$ is concave in $x$. Thus, $f(x) = h(g(x))$ is convex. (We can also show the same by observing that $f''(x) \geq 0$).}
				
				Now note that the sum of convex functions is convex. Therefore, $\sum_{i\in[k]} \frac{A_{1i}}{\sqrt{\widehat{P}(1,i) f_1+ c_i}}$ or the first constraint is convex in $f_1$. Similarly convexity of the other constraints in $f_1$ can be shown. 
			\end{proof}
		}
		\begin{propn}
			\label{propn: ax power -0.5 is convex}
			For any  $a\in\R_+$, the function $g(x) := a x^{-\frac{1}{2}}$ is convex in $x$. 
		\end{propn}
		\begin{proof}
			We observe that the second derivative is positive.
		\end{proof}
		\begin{propn}
			The constraint equations of \hyperref[eqn: optimization problem]{OPT} are convex in $f$
		\end{propn}
		\begin{proof}
			Consider the first constraint of the problem. We can simplify this to get $\sum_{i\in[k]} \frac{A_{1i}}{\sqrt{\widehat{P}(*,i)^\tr f}}$. 
			
			Note that the $i$th term in the summand (i.e,  $\frac{A_{1i}}{\sqrt{\widehat{P}(*,i)^\tr f}}$) is of the form $f(x) = c(v^\tr x)^{-\frac{1}{2}}$ for some $c\in\R_+$ and $v \in \R^{N}_+$. Let $x_1, x_2\in\R^N$ be any two vectors, and scalar $\lambda \in [0,1]$. We wish to show that $f(\lambda x_1 + (1-\lambda)x_2) \leq \lambda f(x_1) + (1-\lambda)f(x_2)$. 
			
			We have $f(\lambda x_1 + (1-\lambda)x_2) = c(v^\tr (\lambda x_1 + (1-\lambda)x_2))^{-\frac{1}{2}} = c(\lambda v^\tr x_1 + (1-\lambda)v^\tr x_2)^{-\frac{1}{2}}$
			
			But $a x^{-\frac{1}{2}}$ is convex as per \hyperref[propn: ax power -0.5 is convex]{Proposition \ref{propn: ax power -0.5 is convex}}. Therefore $c(\lambda v^\tr x_1 + (1-\lambda)v^\tr x_2)^{-\frac{1}{2}} \leq \lambda c (v^\tr x_1)^{-\frac{1}{2}} + (1-\lambda) c (v^\tr x_2)^{-\frac{1}{2}} = \lambda f(x_1) + (1- \lambda) f(x_2)$, as required.
			
			Since $\frac{A_{1i}}{\sqrt{\widehat{P}(*,i)^\tr f}}$ is convex, the sum $\sum_{i\in[k]} \frac{A_{1i}}{\sqrt{\widehat{P}(*,i)^\tr f}}$ is convex as well. Similarly, all the other constraints are also convex.
		\end{proof}
		Since the constraints are convex in $f$ and the objective is linear, \hyperref[eqn: optimization problem]{OPT} is convex.
	\end{proof}

\section{Proof of KL-Divergence Inequality}
\label{appendix section: Proof KL Divergence inequality}
    For completeness, we provide a proof of inequality (\ref{ineq:auer}). 
	
	\begin{lem}
		\label{lem: KL leq 6 epsilon square times number of obs}
		$ \rm{KL} (\P_0,\P_{(a,i)}) \leq 6\beta_i^2 \ \E_{\P_0} [ T_{(a,i)}]$ 
	\end{lem}
	\begin{proof}[Proof of Inequality (\ref{ineq:auer})]
	\label{proof: proposition KL leq 6 epsilon square times number of obs}
		This proof is based on Lemma B1 in \cite{AuerGamblinginaRiggedCasino}. We define a couple of notations for this proof.
		Let $\mathbf{R}_{t-1}$ indicate the filtration (of rewards and other observations) up to time $t-1$. and $R_t$ indicate the reward at time $t$ for this proof.
		\begin{align*}
			 \rm{KL} (\P_0,\P_{(a,i)}) &=  \rm{KL} \left[\prob_{\P_0}(R_T,R_{T-1},\dots,R_1) \mathrel{\Vert} \prob_{\P_{(a,i)}}(R_T,R_{T-1},\dots,R_1)\right]
		\end{align*}
		We now state (without proof) a useful lemma for bounding the KL divergence between random variables over a number of observations.

    			\textbf{Chain Rule for entropy (Theorem 2.5.1 in \cite{JoyCoverBook})}:
    			\label{eqn: Chain rule for entropy}
    			Let $X_1,\dots,X_T$ be random variables drawn according to $P_1,\dots,P_T$. Then 
    			$$H(X_1,X_2,\dots,X_T) = \sum_{i=1}^T H(X_i\mid X_{i-1},X_{i-2},\dots,X_1)$$
    			where $H(\cdot)$ is the entropy associated with the random variables.
    	
		Using the \hyperref[eqn: Chain rule for entropy]{chain rule for entropy}
		\begin{align*}
			 \rm{KL} (\P_0,\P_{(a,i)}) &=\sum\limits_{t=1}^T  \rm{KL} \left[\prob_{\P_0}(R_t \mid \mathbf{R}_{t-1}) \mathrel{\Vert} \prob_{\P_{(a,i)}}(R_t \mid \mathbf{R}_{t-1})\right]\\
			\intertext{Let $a_t$ be the intervention chosen by the Algorithm $\mathscr{A}$ at time $t$. Then:}
			&=\sum\limits_{t=1}^T \prob_{\P_0}\{a_t \neq a \mid \mathbf{R}_{t-1}\} \left(\frac{1}{2} \mathrel{\Vert} \frac{1}{2}\right) + \prob_{\P_0}\{a_t = a \mid \mathbf{R}_{t-1}\} \rm{KL} \left(\frac{1}{2} \mathrel{\Vert} \frac{1}{2} + \beta_i\right) \\
			\intertext{Since $ \rm{KL} \left(\frac{1}{2} \mathrel{\Vert} \frac{1}{2}\right) = 0$, we get:}
			&=\sum\limits_{t=1}^T \prob_{\P_0}\{a_t = a \mid \mathbf{R}_{t-1}\} \rm{KL} \left(\frac{1}{2} \mathrel{\Vert} \frac{1}{2} + \beta_i\right) \\
			&=  \rm{KL} \left(\frac{1}{2} \mathrel{\Vert} \frac{1}{2} + \beta_i\right) \sum\limits_{t=1}^T \prob_{\P_0}\{a_t = a \mid \mathbf{R}_{t-1}\} \\
			&=  \rm{KL} \left(\frac{1}{2} \mathrel{\Vert} \frac{1}{2} + \beta_i\right) \E_{\P_0}[T_{(a,i)}] \\
		\end{align*}
		
		\begin{claim}
			\label{claim: KL leq 6 betai square}
			$ \rm{KL} \left(\frac{1}{2} \mathrel{\Vert} \frac{1}{2}+\beta_i\right)=-\frac{1}{2} \log_2(1-4\beta_i^2)\leq 6\beta_i^2$
		\end{claim}
		\begin{proof}
			\begin{align*}
				 \rm{KL} \left(\frac{1}{2} \mathrel{\Vert} \frac{1}{2}+\beta_i\right) &= \frac{1}{2}\log_2\left[\frac{\frac{1}{2}}{\frac{1}{2} + \beta_i}\right] + (1-\frac{1}{2})\log_2\left[\frac{(1-\frac{1}{2})}{(1-\frac{1}{2} - \beta_i)}\right]\\
				&= \frac{1}{2}\log_2\left[\frac{1}{1 + 2\beta_i}\right] + \frac{1}{2}\log_2\left[\frac{1}{1 - 2\beta_i}\right] = \frac{1}{2}\log_2\left[\frac{1}{1 - 4\beta_i^2}\right]=-\frac{1}{2}\log_2\left[1 - 4\beta_i^2\right]\\
				&=-\frac{1}{2\ln(2)} \ln\left[1 - 4\beta_i^2\right] \leq \frac{4\beta_i^2}{2\ln(2)} < 6\beta_i^2
			\end{align*}
			where the last inequality is obtained from the Taylor series expansion of the $\log$.
		\end{proof}
		It follows that: 
		$ \rm{KL} (\prob_0,\prob_1) \leq 6\beta_i^2\E_{\P_0}[T_{(a,i)}]$.
	\end{proof}

\end{document}